\documentclass[11pt]{article}
\pdfoutput=1
\usepackage{etoolbox,verbatim}
\newtoggle{arxiv}
\toggletrue{arxiv}
\newtoggle{colt}
\togglefalse{colt}
\newtoggle{iclr}
\togglefalse{iclr}
\newtoggle{customthms}
\toggletrue{customthms}

\usepackage[table,dvipsnames]{xcolor}
\usepackage{tcolorbox}
\tcbuselibrary{skins,breakable}

\tcbset{
  aibox/.style={
      width=\linewidth,
      top=6pt,
      bottom=0pt,
      left=1.5pt,
      right=1.5pt,
colback=blue!60!gray!5,
colframe=black,
colbacktitle=black,
      enhanced,
      center,
      attach boxed title to top left={yshift=-0.1in,xshift=0.15in},
      boxed title style={boxrule=0pt,colframe=white,},
    }
}
\newtcolorbox{AIbox}[2][]{aibox,title=#2,#1}

\definecolor{primalcolor}{HTML}{A60000}
\definecolor{contrarycolor}{HTML}{00A6A6}
\definecolor{darkcontrarycolor}{HTML}{004C4C}
\definecolor{lightblue}{HTML}{2970CC}
\definecolor{lightpurple}{HTML}{673147}
\definecolor{ForestGreen}{HTML}{FF5733}
\definecolor{myred}{HTML}{AA4A44}
\definecolor{hyppurple}{HTML}{800080}

\newcommand{\linkcolor}{darkcontrarycolor}
\newcommand{\urlcolor}{darkcontrarycolor}
\newcommand{\citecolor}{darkcontrarycolor}

\newcommand{\thmcolordark}{black}

\usepackage[T1]{fontenc}
\usepackage{capt-of}
\usepackage[font=small,labelfont=bf]{caption}

\usepackage{natbib}
\usepackage{amssymb,upgreek,bm}
\usepackage{mathtools}
\usepackage[english]{babel}  
\usepackage{multirow,tcolorbox,tikz-cd,turnstile,xspace,xparse}
\usepackage{relsize,breakcites,appendix}
\usepackage{longtable,tablefootnote,booktabs,float,makecell}
\usepackage[normalem]{ulem}
\usepackage{tabu}
\usepackage[shortlabels]{enumitem} 
\usepackage{footnote}
\usepackage{boxedminipage}
\usepackage{multirow,nicefrac}
\usepackage{verbatim} %
\usepackage{wrapfig}

\usepackage[colorlinks=true,linkcolor=\linkcolor,urlcolor=\urlcolor,citecolor=\citecolor,breaklinks]{hyperref}

\usepackage{prettyref}

\usepackage[capitalise,nameinlink]{cleveref}

\iftoggle{colt}{
    \DeclareRobustCommand{\qed}{
        \usepackage{thmtools}
          \ifmmode \mathqed
          \else
            \leavevmode\unskip\penalty9999 \hbox{}\nobreak\hfill
            \quad\hbox{\qedsymbol}%
          \fi
    }
}
{
    \usepackage{amsthm}
     
}
\usepackage{thm-restate}

\iftoggle{arxiv}
{
    \usepackage{mathrsfs}
    \usepackage[margin=2cm]{geometry}
    \usepackage[bitstream-charter]{mathdesign}
     
    \usepackage[scaled=0.92]{PTSans}

    \usepackage{subfig}

    \usepackage{fullpage}
    \usepackage[noend]{algpseudocode}
    \usepackage{algorithm}
   
}
{
    \usepackage{mathrsfs}
}

\DeclareMathAlphabet{\mathbfsf}{\encodingdefault}{\sfdefault}{bx}{n}

\numberwithin{equation}{section}

\iftoggle{arxiv}
{
    }
{
    
}
\iftoggle{arxiv}
{
    
}
{
    
}

\usepackage{thmtools}

\Crefname{equation}{Eq.}{Eqs.}
\Crefname{assumption}{Assumption}{Assumptions}
\Crefname{condition}{Condition}{Conditions}
\Crefname{claim}{Claim}{Claims}
\Crefname{property}{Property}{Properties}
\Crefname{construction}{Construction}{Constructions}

\declaretheoremstyle[
    headformat=\normalfont\textcolor{\thmcolordark}{\bfseries\NAME\,\NUMBER}\NOTE,%
    notefont={\normalfont\textcolor{\thmcolordark}{\bfseries}}, 
    notebraces={}{},
    bodyfont=\normalfont\itshape,
    spaceabove = 6pt,
    spacebelow = 6pt,
    ]{coloredthmversion}

\declaretheoremstyle[
    headformat=\normalfont\textcolor{\thmcolordark}{\bfseries\NAME\,\NUMBER}\NOTE,%
    bodyfont=\normalfont\itshape,
    spaceabove = 6pt,
    spacebelow = 6pt,
    ]{coloredthm}

\declaretheoremstyle[
    headformat=\normalfont\textcolor{\thmcolordark}{\bfseries\NAME\,\NUMBER}\NOTE,%
    bodyfont=\normalfont,
    spaceabove = 6pt,
    spacebelow = 6pt,
    ]{coloreddef}

\iftoggle{customthms}
{
    \theoremstyle{coloredthmversion}
}
{}

\iftoggle{customthms}{
  \theoremstyle{coloredthm}
  \newtheorem{theorem}{Theorem}
  \newtheorem{lemma}{Lemma}
  \newtheorem{corollary}{Corollary}
  \newtheorem{proposition}{Proposition}
}
{}

\newtheorem*{thminformal*}{Informal Theorem}

\iftoggle{customthms}{
    \theoremstyle{coloreddef}
    \newtheorem{definition}{Definition}

    \newtheorem{property}{Property}
    
}
{
  
}

\newtheorem{assumption}{Assumption}[section]
\newtheorem{condition}{Condition}[section]

\makeatletter
\newcommand{\neutralize}[1]{\expandafter\let\csname c@#1\endcsname\count@}
\makeatother

\iftoggle{customthms}{
    \newtheoremstyle{named}{}{}{\itshape}{}{\bfseries}{}{.5em}{\Cref{#3} {\normalfont (informal)} }{}
    \theoremstyle{named}
    
    \theoremstyle{plain}
}
{}

\newtheorem*{theorem*}{Theorem}
\newtheorem*{lemma*}{Lemma}
\newtheorem*{corollary*}{Corollary}
\newtheorem*{proposition*}{Proposition}
\newtheorem*{claim*}{Claim}
\newtheorem*{fact*}{Fact}
\newtheorem*{observation*}{Observation}
\newtheorem*{definition*}{Definition}
\newtheorem*{remark*}{Remark}
\newtheorem*{example*}{Example}

\def\ddefloop#1{\ifx\ddefloop#1\else\ddef{#1}\expandafter\ddefloop\fi}
\def\ddef#1{\expandafter\def\csname bb#1\endcsname{\ensuremath{\mathbb{#1}}}}
\ddefloop ABCDEFGHIJKLMNOPQRSTUVWXYZ\ddefloop

\def\ddefloop#1{\ifx\ddefloop#1\else\ddef{#1}\expandafter\ddefloop\fi}
\def\ddef#1{\expandafter\def\csname frak#1\endcsname{\ensuremath{\mathfrak{#1}}}}
\ddefloop ABCDEFGHIJKLMNOPQRSTUVWXYZ\ddefloop

\def\ddefloop#1{\ifx\ddefloop#1\else\ddef{#1}\expandafter\ddefloop\fi}
\def\ddef#1{\expandafter\def\csname fr#1\endcsname{\ensuremath{\mathfrak{#1}}}}
\ddefloop ABCDEFGHIJKLMNOPQRSTUVWXYZ\ddefloop

\def\ddefloop#1{\ifx\ddefloop#1\else\ddef{#1}\expandafter\ddefloop\fi}
\def\ddef#1{\expandafter\def\csname eul#1\endcsname{\ensuremath{\EuScript{#1}}}}
\ddefloop ABCDEFGHIJKLMNOPQRSTUVWXYZ\ddefloop

\def\ddefloop#1{\ifx\ddefloop#1\else\ddef{#1}\expandafter\ddefloop\fi}
\def\ddef#1{\expandafter\def\csname scr#1\endcsname{\ensuremath{\mathscr{#1}}}}
\ddefloop ABCDEFGHIJKLMNOPQRSTUVWXYZ\ddefloop

\def\ddefloop#1{\ifx\ddefloop#1\else\ddef{#1}\expandafter\ddefloop\fi}
\def\ddef#1{\expandafter\def\csname b#1\endcsname{\ensuremath{\mathbf{#1}}}}
\ddefloop ABCDEFGHIJKLMNOPQRSTUVWXYZ\ddefloop

\def\ddefloop#1{\ifx\ddefloop#1\else\ddef{#1}\expandafter\ddefloop\fi}
\def\ddef#1{\expandafter\def\csname bhat#1\endcsname{\ensuremath{\hat{\mathbf{#1}}}}}
\ddefloop ABCDEFGHIJKLMNOPQRSTUVWXYZ\ddefloop

\def\ddefloop#1{\ifx\ddefloop#1\else\ddef{#1}\expandafter\ddefloop\fi}
\def\ddef#1{\expandafter\def\csname btil#1\endcsname{\ensuremath{\tilde{\mathbf{#1}}}}}
\ddefloop ABCDEFGHIJKLMNOPQRSTUVWXYZ\ddefloop

\def\ddefloop#1{\ifx\ddefloop#1\else\ddef{#1}\expandafter\ddefloop\fi}
\def\ddef#1{\expandafter\def\csname bst#1\endcsname{\ensuremath{\mathbf{#1}^\star}}}
\ddefloop ABCDEFGHIJKLMNOPQRSTUVWXYZ\ddefloop

\def\ddefloop#1{\ifx\ddefloop#1\else\ddef{#1}\expandafter\ddefloop\fi}
\def\ddef#1{\expandafter\def\csname bst#1\endcsname{\ensuremath{\mathbf{#1}^\star}}}
\ddefloop abcdeghijklmnopqrstuvwxyz\ddefloop

\def\ddefloop#1{\ifx\ddefloop#1\else\ddef{#1}\expandafter\ddefloop\fi}
\def\ddef#1{\expandafter\def\csname bhat#1\endcsname{\ensuremath{\hat{\mathbf{#1}}}}}
\ddefloop abcdefghijklmnopqrstuvwxyz\ddefloop

\def\ddefloop#1{\ifx\ddefloop#1\else\ddef{#1}\expandafter\ddefloop\fi}
\def\ddef#1{\expandafter\def\csname b#1\endcsname{\ensuremath{\mathbf{#1}}}}
\ddefloop abcdeghijklnopqrstuvwxyz\ddefloop

\def\ddefloop#1{\ifx\ddefloop#1\else\ddef{#1}\expandafter\ddefloop\fi}
\def\ddef#1{\expandafter\def\csname barb#1\endcsname{\ensuremath{\bar{\mathbf{#1}}}}}
\ddefloop abcdefghijklmnopqrstuvwxyz\ddefloop

\def\ddef#1{\expandafter\def\csname c#1\endcsname{\ensuremath{\mathcal{#1}}}}
\ddefloop ABCDEFGHIJKLMNOPQRSTUVWXYZ\ddefloop
\def\ddef#1{\expandafter\def\csname h#1\endcsname{\ensuremath{\widehat{#1}}}}
\ddefloop ABCDEFGHIJKLMNOPQRSTUVWXYZ\ddefloop
\def\ddef#1{\expandafter\def\csname hc#1\endcsname{\ensuremath{\widehat{\mathcal{#1}}}}}
\ddefloop ABCDEFGHIJKLMNOPQRSTUVWXYZ\ddefloop
\def\ddef#1{\expandafter\def\csname t#1\endcsname{\ensuremath{\widetilde{#1}}}}
\ddefloop ABCDEFGHIJKLMNOPQRSTUVWXYZ\ddefloop
\def\ddef#1{\expandafter\def\csname tc#1\endcsname{\ensuremath{\widetilde{\mathcal{#1}}}}}
\ddefloop ABCDEFGHIJKLMNOPQRSTUVWXYZ\ddefloop

\newcommand{\email}[1]{\texttt{#1}}

\Crefname{component}{Component}{Components}
\Crefname{contribution}{Contribution}{Contributions}
\newcommand{\componentref}[1]{%
  \hyperref[#1]{C\ref*{#1}}%
}
\Crefname{claim}{Claim}{Claims}
\Crefname{property}{Property}{Properties}

\iftoggle{iclr}
{
  
}
{
  
}

\makeatletter
\newcommand\addtometadatalist[5][]{%
  \begingroup
  \if\relax#3\relax\def\sep{}\else\def\sep{#5}\fi
  \let\protect\@unexpandable@protect
  \xdef#3{\expandafter{#3}\sep #4[#1]{#2}}%
  \endgroup
}

\newcommand\metadatalist{}
\newcommand\metadataformat[2][]{{\small \textbf{#1:} #2}}
\newcommand\metadata[2][]{\addtometadatalist[#1]{#2}{\metadatalist}{\metadataformat}{\\}}
\makeatother

\newcommand{\paperwebsite}[1]{\metadata[Website]{\url{#1}}}
\newcommand{\papercode}[1]{\metadata[Code]{\url{#1}}}
\newcommand{\paperdocs}[1]{\metadata[Documentation]{\url{#1}}}
\newcommand{\paperblog}[1]{\metadata[Blog]{\url{#1}}}

\usepackage{preamble/color-edits}

\iftoggle{arxiv}
{
\input{preamble/arxiv_title}
}
{}

\usetikzlibrary{calc}
\addauthor{ms}{magenta}
\addauthor{cp}{blue}
\addauthor{nch}{cyan}
\addauthor{ga}{pink}
\addauthor{gs}{green}

\usepackage{graphicx}
\usepackage{booktabs}
\usepackage[accsupp]{axessibility}

\usepackage{orcidlink}

\newcommand{\dd}{\mathrm{d}}

\newcommand{\grad}{\mathrm{grad}}

\newcommand{\argmin}{\operatornamewithlimits{argmin}}

\renewcommand{\first}[1]{\mathbf{#1}}
\renewcommand{\second}[1]{\underline{#1}}

\usepackage{float}
\usepackage{enumitem}

\usepackage{amsmath}
\usepackage{amssymb}
\usepackage{amsfonts}
\usepackage{bm}

\usepackage{algorithm}
\usepackage{algpseudocode}
\algrenewcomment[1]{\hfill$\triangleright$~#1}
\algrenewcommand\algorithmicrequire{\textbf{Input:}}
\algrenewcommand\algorithmicensure{\textbf{Return:}}

\usepackage{multirow}
\usepackage{booktabs}
\usepackage{graphicx}
\usepackage{subcaption}
\usepackage{arydshln}

\usepackage{wrapfig}

\usepackage[table, dvipsnames]{xcolor}

\usepackage{thmtools}
\usepackage{thm-restate}

\usepackage{titletoc}

\tcolorboxenvironment{definition}{
  colback=blue!3,
  colframe=blue!3,
  boxrule=0.5pt,
  arc=6pt,
  left=6pt,
  right=6pt,
  top=5pt,
  bottom=5pt,
  before skip=6pt,
  after skip=6pt
}

\tcolorboxenvironment{proposition}{
  colback=blue!3,
  colframe=blue!3,
  boxrule=0.5pt,
  arc=6pt,
  left=6pt,
  right=6pt,
  top=5pt,
  bottom=5pt,
  before skip=6pt,
  after skip=6pt
}

\tcolorboxenvironment{theorem}{
  colback=blue!3,
  colframe=blue!3,
  boxrule=0.5pt,
  arc=6pt,
  left=6pt,
  right=6pt,
  top=5pt,
  bottom=5pt,
  before skip=6pt,
  after skip=6pt
}

\definecolor{DarkGreen}{rgb}{0.00, 0.40, 0.00}
\definecolor{RoyalBlue}{rgb}{0.25, 0.35, 0.74}
\definecolor{Crimson}{rgb}{0.86, 0.08, 0.24}
\definecolor{ChromeYellow}{rgb}{1.0, 0.65, 0.0}

\begin{document}

\title{DiffRGD: An Inference-Time Diffusion Guidance \\
Through Riemannian Gradient Descent}

\author{
Jia-Wei Liao\textsuperscript{\rm 1,\rm 4}, 
Li-Xuan Peng\textsuperscript{\rm 2}, \\
Mei-Heng Yueh\textsuperscript{\rm 3},
Min Sun\textsuperscript{\rm 2},
Cheng-Fu Chou\textsuperscript{\rm 1},
Jun-Cheng Chen\textsuperscript{\rm 4} \vspace{0.3em} \\
\small
\textsuperscript{\rm 1}National Taiwan University,
\textsuperscript{\rm 2}National Tsing Hua University, \\
\small
\textsuperscript{\rm 3}National Taiwan Normal University,
\textsuperscript{\rm 4}Academia Sinica \\
\small \email{d11922016@csie.ntu.edu.tw, pullpull@citi.sinica.edu.tw}
}

\begin{tcolorbox}[
    colback=blue!5, colframe=blue!5,
    boxrule=0pt,
    arc=2mm%
  ]
  \centering
  \maketitle
  \vspace{-1em}
  \begin{abstract}
Recently, diffusion models have been widely adopted in generative modeling and have served as foundational models for many image generation tasks. To control the generation without costly re-training or fine-tuning, many works seek inference-time guidance methods to steer the latent via a differentiable objective at inference time. However, these methods cannot effectively preserve the original Gaussian distribution because they introduce distributional drift, thereby degrading the sample quality. To address this gap, we propose DiffRGD, a distribution-aware guidance framework that explicitly preserves the latent Gaussian structure. DiffRGD formulates each sampling step as a constrained optimization problem on a spherical manifold induced by the latent Gaussian distribution, and solves it efficiently via Riemannian Gradient Descent (RGD). DiffRGD is a plug-and-play method that can be seamlessly integrated into any pre-trained diffusion model. Extensive experiments demonstrate that DiffRGD outperforms previous methods in most image restoration and conditional generation tasks. Our project page is available at \url{https://diffrgd.github.io/}.
\end{abstract}

  \vspace{-1em}
  \vskip 0.5cm
  \begin{flushright}
    \includegraphics[width=0.05\linewidth]{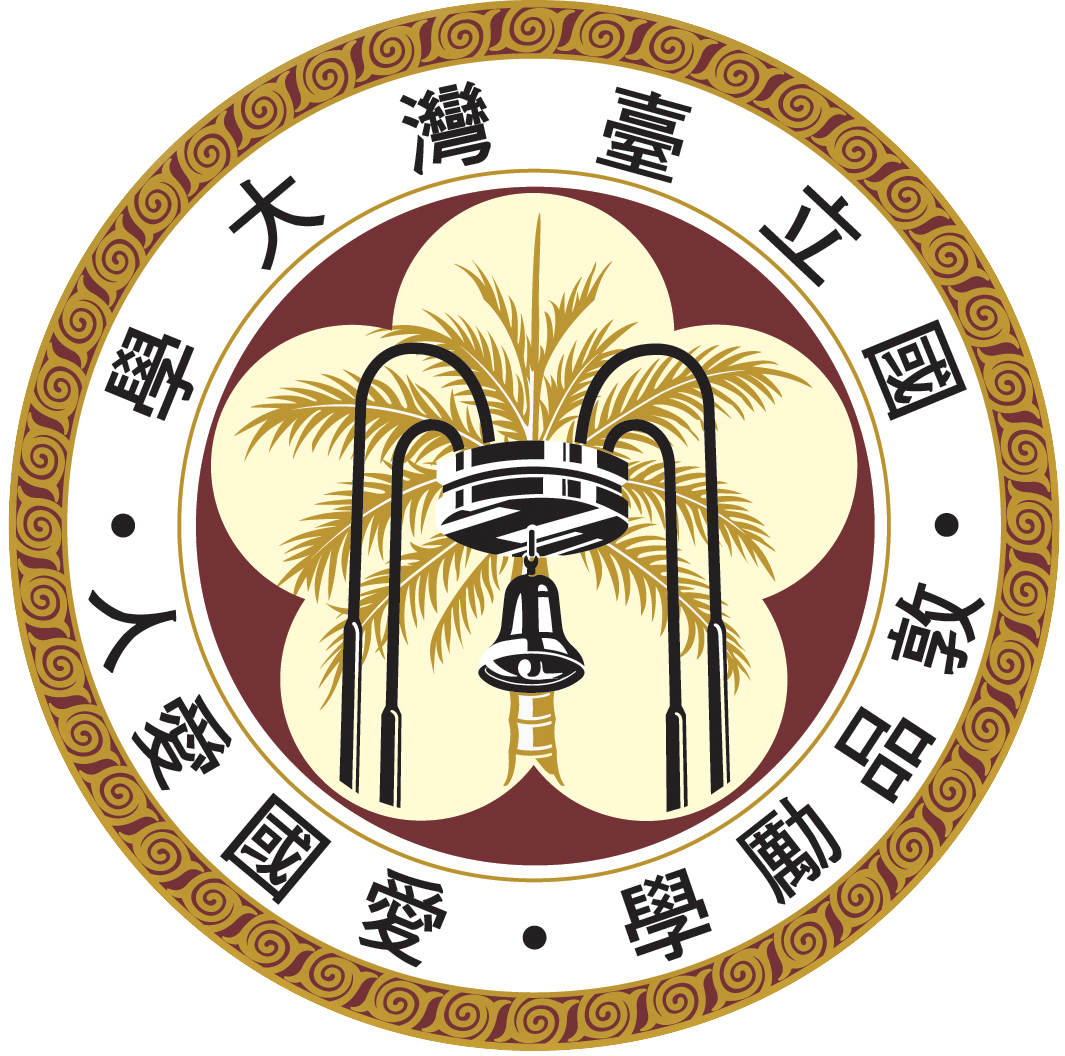}
    \hspace{0.01em}
    \includegraphics[width=0.05\linewidth]{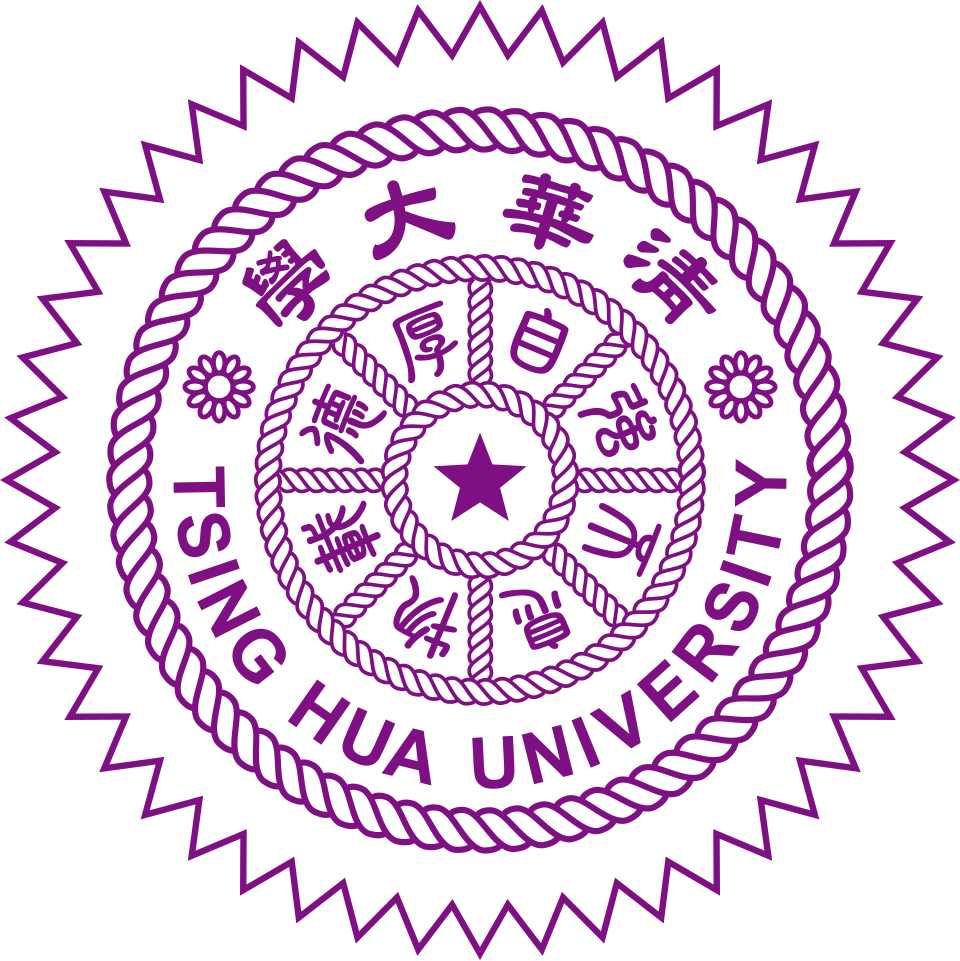}
    \hspace{0.01em}
    \includegraphics[width=0.05\linewidth]{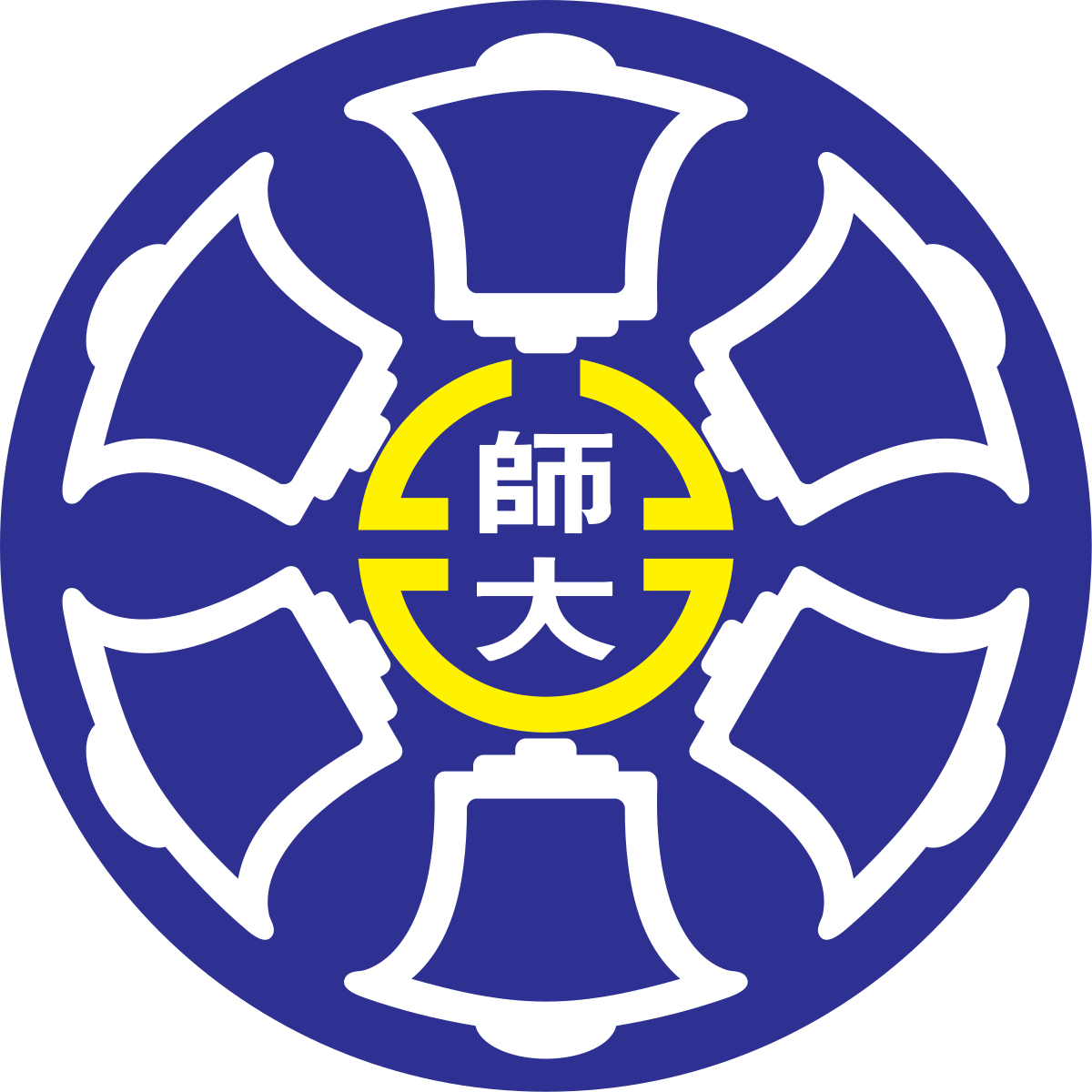}
    \hspace{0.01em}
    \includegraphics[width=0.05\linewidth]{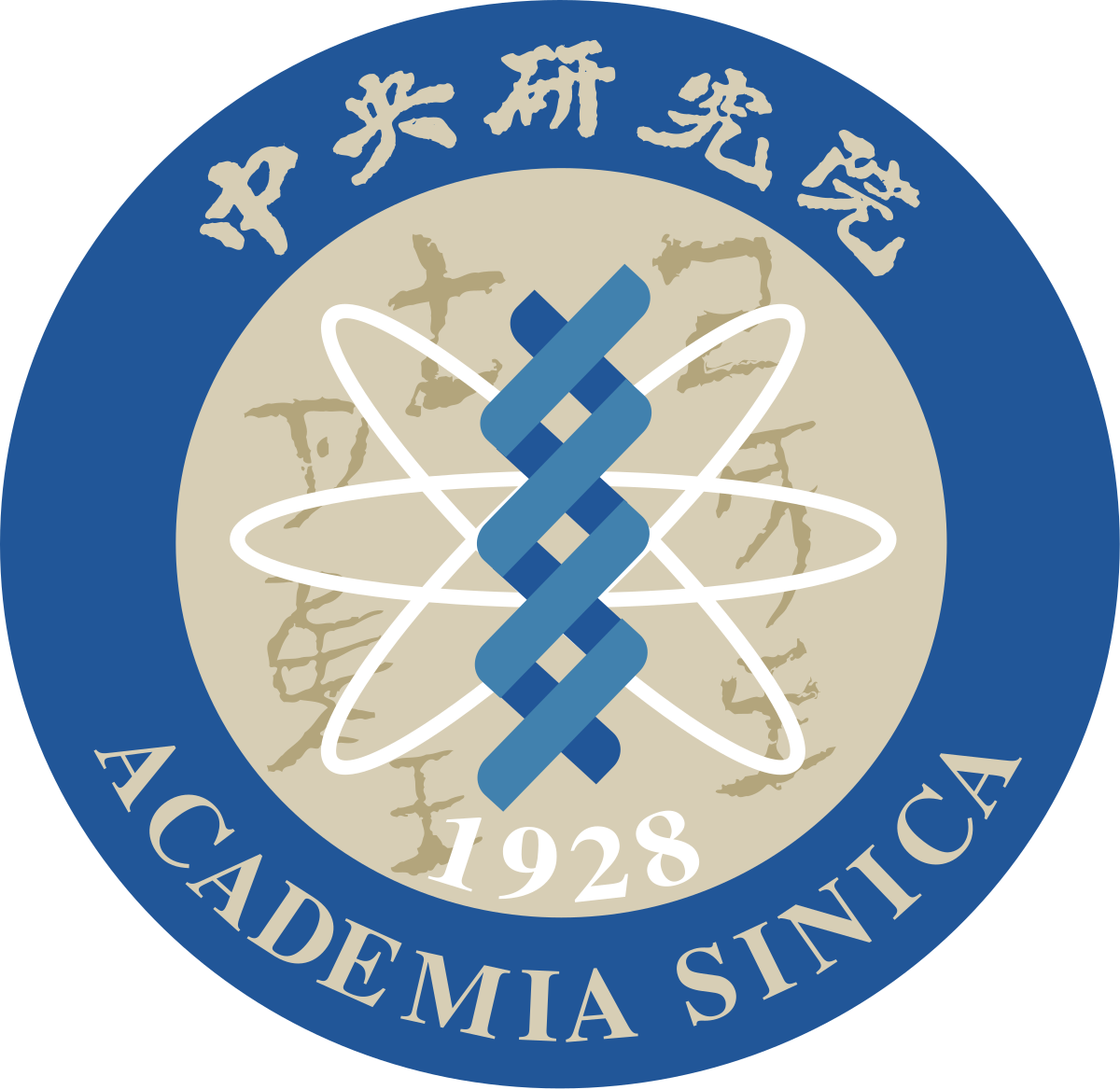}
  \end{flushright}
  \makeatletter
  \ifdefempty{\metadatalist}{}{\metadatalist\par}
  \makeatother
\end{tcolorbox}

\section{Introduction}

Diffusion models~\citep{sohl2015pdm,ho2020ddpm,song2021score,lai2025diffusion} have achieved state‑of‑the‑art performance in unconditional and conditional generation. Considering the prohibitively expensive training efforts of large conditional diffusion models~\citep{ho2021cfg, zhang2023controlnet}, many works seek inference-time guidance, where the sampler is steered by differentiable objectives without any re-training or fine-tuning. Most existing methods~\citep{chung2023dps, he2024mpgd, yang2024dsg} inject external gradients directly into the sampling process ~\citep{dhariwal2021cg}. Specifically, they formulate the guidance for each step as an optimization problem:
\begin{equation}
    \label{eq:general_prob_formulation}
    \begin{aligned}
        \hat{\bm{x}}_{t-1} & \sim p_\theta(\bm{x}_{t-1} | \bm{x}_{t}), \\
        \bm{x}_{t-1} &= \argmin_{\bm{x} \in \mathcal{N}(\hat{\bm{x}}_{t-1})} \mathcal{L}(\hat{\bm{x}}_0(\bm{x}, t-1), \bm{y}).
    \end{aligned}
\end{equation}
Here $\hat{\bm{x}}_{t-1}$ denotes the noisy latent sampled from the reverse distribution $p_\theta(\bm{x}_{t-1} | \bm{x}_{t})$ parameterized by the model $\theta$.
This prediction serves as the reference point for constructing the neighborhood constraint $\mathcal{N}(\hat{\bm{x}}_{t-1})$, which restricts the optimization process as it moves in the guidance direction induced by the guidance loss $\mathcal{L}(\cdot, \bm{y})$. Note that $\hat{\bm{x}}_0(\bm{x}, t-1)$ denotes the model’s estimate of the clean sample (i.e., at timestep $t=0$ obtained from timestep $t-1$); this estimate is used to compute the guidance loss.

\begin{wrapfigure}{r}{0.4\linewidth}
    \centering
    \includegraphics[width=\linewidth]{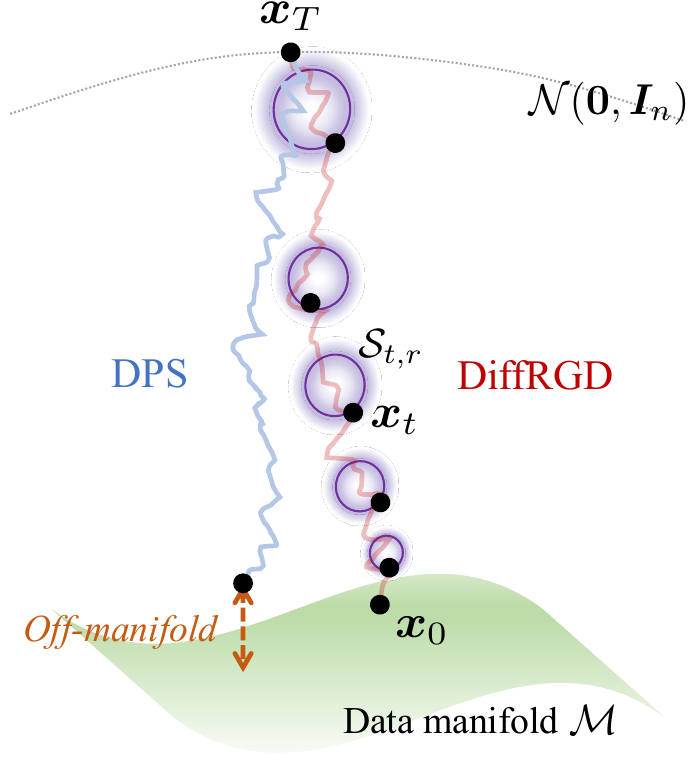}
    \caption{Illustration of our method: Previous works perform guidance without latent-distribution-aware constraints, which can result in the off-manifold phenomena that hinder the sample quality; DiffRGD proposes a latent-distribution-aware geometry constraint tailored to the Gaussian properties of the diffusion model.
    }
    \vspace{-10pt}
    \label{fig:teaser}
\end{wrapfigure}

However, their guidance operates in the ambient space that cannot preserve the stepwise Gaussian latent distribution. This might elicit distributional drifts induced by the Jensen gap~\citep{chung2023dps,yang2024dsg}, which ultimately degrade the clean sample quality. A representative example is DPS~\citep{chung2023dps}, which applies gradients from the guiding loss without any constraints. When the step size is large, the deviation can be exacerbated, whereas a small step size might lead to inadequate controllability (Figure~\ref{fig:teaser}).

These research gaps motivate us to ask: \emph{Can we design latent-distribution-aware guidance that ensures the guided sample lies within the original latent distribution?}

In this paper, we propose \textit{DiffRGD}, an inference-time diffusion guidance framework based on \emph{Riemannian optimization}~\citep{absil2008rgd}. We follow the problem formulation of previous works defined in Equation~\ref{eq:general_prob_formulation} with our proposed novel geometry constraints. Our key insight is that since each DDIM sampling timestep~\citep{song2021ddim} defines an \emph{isotropic Gaussian} of latent distribution, we can construct a suitable constraint that respects the Gaussian properties to perform Riemannian Gradient Descent (RGD). 
DiffRGD constructs the constraint as a spherical manifold induced from the latent distribution via our proposed polar decomposition property. We also prove that our method converges to a stationary point. In addition, akin to previous methods, DiffRGD is a plug-and-play method and can be seamlessly applied to any pre-trained diffusion model.

We conduct extensive experiments, including synthetic data and real-world tasks. In real-world tasks, DiffRGD outperforms previous methods in most image restoration~\citep{lugmayr2022repaint, kawar2022ddrm, saharia2023sr3, wang2023ddnm, zhu2023diffpir, song2023pigdm, chung2022mcg, chung2023dps, he2024mpgd, yang2024dsg, zhang2025admmdiff} (e.g., inpainting, super-resolution, deblurring, denoising) and conditional generation tasks~\citep{yu2023freedom, bansal2024ugd, song2023lgd, he2024mpgd, yang2024dsg, zhang2025admmdiff} (e.g., segmentation map, sketch, FaceID, style). Additionally, we reimplement all the compared methods under the Diffusers framework~\citep{platen2022diffusers}.

Our main contributions are summarized as follows:
\begin{enumerate}
    \item We identify a polar-decomposition property that reveals a distribution-induced geometry for diffusion guidance. Leveraging this insight, we propose a Riemannian guidance framework that respects the latent distribution geometry and mitigates sample degradation caused by off-manifold distribution drift.

    \item We provide an analysis of distribution preservation to validate the effectiveness of the proposed method and prove the convergence of the algorithm. In practice, the proposed geometry-aware optimization converges faster than ADMM-based methods.
    
    \item We demonstrate the effectiveness of DiffRGD across four image restoration tasks and three conditional generation tasks, achieving state-of-the-art performance without re-training or fine-tuning.

\end{enumerate}

\section{Related Work}

\subsection{Diffusion Models for Conditional Generation}
Diffusion models have emerged as a dominant class of generative models due to their strong modeling capacity and sample quality~\citep{sohl2015pdm, ho2020ddpm, song2021score, song2021ddim, karras2022edm}. Recent advancements have enabled scalable and foundational implementations, such as Stable Diffusion~\citep{rombach2022ldm, esser2024sd3} and FLUX~\citep{flux, flux.1-dev-weights}, which have been widely adopted in commercial applications. To enable controllable generation, several guidance strategies have been introduced to steer the sampling process. 

Classifier guidance (CG)~\citep{dhariwal2021cg} utilizes an external classifier to steer the sampling process, while classifier-free guidance (CFG)~\citep{ho2021cfg} achieves guidance by interpolating conditional and unconditional model outputs without requiring an external classifier. These techniques have been widely adopted across diverse downstream applications, including text-to-image synthesis~\citep{rombach2022ldm}, conditional image generation~\citep{nichol2022glide, voynov2023sketchguided, zhang2023controlnet, yu2023freedom, song2023lgd, bansal2024ugd, wu2024text2qr, liao2024diffqrcoder}, and test-time adaptation~\citep{tsai2024gda}.

\subsection{Inference-Time Diffusion Guidance}
Despite the powerful capability of diffusion models, their training demands substantial computational resources and large-scale datasets. Fine-tuning or re-training diffusion models for every new downstream task or modality is prohibitively expensive and leads to high maintenance overhead. To address this, several works have proposed inference-time guidance methods that refine the latent during inference time without modifying the original model parameters.

\paragraph{DPS.}
DPS~\citep{chung2023dps} derives the guidance by maximizing posterior likelihood from the inverse problem, and FreeDoM~\citep{yu2023freedom} derives the guidance by minimizing an energy function that measures the conditional alignment. They both share a similar update rule that does not constrain the updates to respect the latent distribution. Although they are intuitive and effective, they often suffer from off-manifold updates if the step size is too large or insufficient guidance when the step size is too small, both of which can degrade the final clean sample quality.

\paragraph{MPGD.}
To alleviate the off-manifold issue, MPGD~\citep{he2024mpgd} proposes a Projected Gradient Descent (PGD) framework with a linear manifold hypothesis. At each timestep, it estimates the clean data $\bm{x}_0$, projects the guiding gradient onto the tangent space of the data manifold using a VAE encoder, and updates the latent accordingly. While this projection mitigates deviation from the data manifold, it assumes local linearity and introduces instability due to the reliance on a perfect encoder, which in practice is infeasible. This limits its general applicability.

\paragraph{DSG.}
Alternatively, DSG~\citep{yang2024dsg} leverages a spherical Gaussian structure and projects the guided latent to a fixed radius hypersphere. 
However, their proposed hypersphere cannot fully adhere to the latent distribution's Gaussian properties, even degrading to a uniform distribution. In contrast, we propose a polar decomposition directly from the latent and a radius sampling technique to maintain the Gaussian properties for the optimization constraints.

\paragraph{ADMMDiff.}
Another line of work, such as ADMMDiff~\citep{zhang2025admmdiff}, formulates diffusion guidance as a constrained optimization problem and incorporates the Alternating Direction Method of Multipliers (ADMM)~\citep{gabay1976admm, chan2016pnpadmm, wang2017parameter} into the sampling process. While theoretically grounded, it utilizes the proximal operator and inner-loop iteration, leading to slow convergence. It also requires careful tuning of multiple hyperparameters, making it less flexible in various applications.

We seek to address the research gap by proposing a more suitable geometric constraint that is directly induced from the original latent distribution and solving it via Riemannian Gradient Descent~\citep{absil2008rgd}, which provides efficient computation with a convergence guarantee.

\section{Method}

Denoising Diffusion Probabilistic Models (DDPMs)~\citep{ho2020ddpm} generate data by learning to reverse a forward noising process, which can equivalently be viewed as learning to sample across different levels of noise-perturbed data distributions~\citep{song2019ncsn, song2021score}. DDIM~\citep{song2021ddim} generalized the DDPM formulation with parameters $\sigma_t$ controlling the noise level during sampling to enable sample acceleration and a unified formulation. Here, we follow their formulation to explain our motivation. Starting from a noisy latent $\bm{x}_t$ at timestep $t$, DDIM defines the next-step latent $\bm{x}_{t-1}$ during sampling, i.e., denoising, as sampling from a Gaussian distribution parameterized by the model prediction. 
\begin{equation*}
    \bm{x}_{t-1} \sim \mathcal{N}(\bm{\mu}_t, \sigma_t^2 \bm{I}_n),
\end{equation*}
where the mean is defined by
\begin{equation*}
    \bm{\mu}_t := \sqrt{\bar{\alpha}_{t-1}} \hat{\bm{x}}_0(\bm{x}_t, t) + \sqrt{1 - \bar{\alpha}_{t-1} - \sigma_t^2} \bm{\epsilon}_\theta(\bm{x}_t, t),
\end{equation*}
with given DDIM sampling schedule parameters $\{\bar{\alpha}_t\}_{t=1}^T$
and the $\hat{\bm{x}}_0(\bm{x}_t, t)$ is the predicted clean sample, estimated by Tweedie’s formula~\citep{efron2011tweedie}:
\begin{equation*}
    \hat{\bm{x}}_0(\bm{x}_t, t) := \frac{1}{\sqrt{\bar{\alpha}_t}} \left( \bm{x}_t - \sqrt{1 - \bar{\alpha}_t} \bm{\epsilon}_\theta(\bm{x}_t, t) \right),
\end{equation*}
in which the $\bm{\epsilon}_\theta(\bm{x}_t, t)$ denotes the predicted noise from a
diffusion model.

In the unconditional sampling settings, the latents across sampling timesteps will form an isotropic Gaussian coupled with the timestep $t$. However, when applying inference-time conditional sampling, the guidance often deviates the latent from the original Gaussian distribution. As the sampling step proceeds, the deviation will accumulate, making the final clean sample deviate from the data manifold, resulting in lower quality. To address the off-manifold problem, we propose DiffRGD, an inference-time guidance method based on Riemannian Gradient Descent (RGD)~\citep{absil2008rgd}. Our key insight is that for each sampling timestep, the latent follows a time-dependent Gaussian distribution, and by the properties of an isotropic Gaussian, we can construct a distribution-induced spherical manifold as a proxy to perform RGD while ensuring the optimized latent still lies within the original Gaussian distribution. The list of notations used in this section is provided in the Appendix~\ref{appendix:notation}.

\subsection{Polar Decomposition of Isotropic Gaussian}
In this section, we characterize the geometric structure induced by the latent distribution at each diffusion timestep. We begin by showing that the isotropic Gaussian latent admits a polar decomposition into statistically independent radial and directional components, as stated in the following proposition.

\begin{proposition}[Polar Decomposition of Isotropic Gaussian]~\label{props:polar_decom}
At each diffusion timestep $t-1$, the latent $\bm{x}_{t-1} \sim \mathcal{N}(\bm{\mu}_t, \sigma_t^2 \bm{I}_n)$ follows an isotropic Gaussian distribution centered at $\bm{\mu}_t$. It admits a polar decomposition:
\begin{equation*}
\bm{x}_{t-1} = \bm{\mu}_t + \sigma_t r \bm{u},
\end{equation*}
where $r \sim \chi(n)$ is a chi-distributed random variable with $n$ degrees of freedom, and $\bm{u} \sim \operatorname{Unif}(\mathbb{S}^{n-1})$ is a unit vector uniformly distributed on the sphere. $r$ and $\bm{u}$ are statistically independent.
\end{proposition}

\begin{wrapfigure}{r}{0.42\linewidth}
    \centering
    \includegraphics[width=0.95\linewidth]{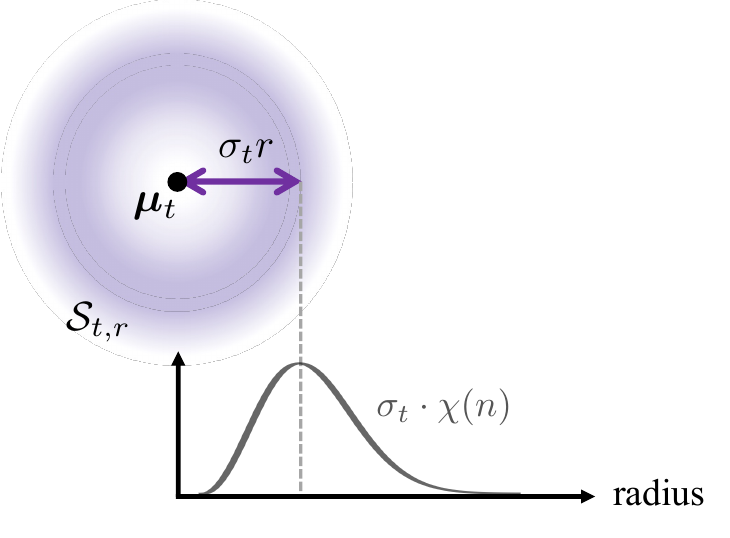}
    \vspace{-10pt}
    \caption{
    Gaussian distribution-aware spherical manifold construction.
    }
    \label{fig:chi_dist}
    \vspace*{-15pt}
\end{wrapfigure}

The insight of Proposition~\ref{props:polar_decom} is that we aim to construct the spherical manifold based on the same density as the original Gaussian of $\bm{x}_{t-1}$. The radius $\sigma_tr$ follows a scaled-chi distribution centered close to $\sqrt{n}\sigma_t$ with two-sided probability decreasing. We show the conceptual illustration in Figure~\ref{fig:chi_dist}. We also provide an asymptotic analysis for our scaled-chi distribution behavior when $n$ is large in the Appendix~\ref{appendix:math}.

This polar decomposition decouples the sampled $r$ and direction $\bm{u}$. Here, we propose to fix the sampled $r$ to construct the spherical manifold as the feasible set for RGD. The goal is to find a suitable direction $\bm{u}$ toward the guidance. Specifically, the spherical manifold is defined as:
\begin{equation}~\label{eq:spherical_manifold}
    \mathcal{S}_{t, r} := \left\{ \bm{x} \in \mathbb{R}^n \mid \| \bm{x} - \bm{\mu}_t \|_2 = \sigma_t r \right\}.
\end{equation}

In contrast to previous unconstrained gradient updates~\citep{chung2023dps, yu2023freedom}, our constraint could preserve the distribution-induced geometry of the diffusion process while allowing the directional refinement toward the guidance direction.

\paragraph{Problem Formulation.}
We follow previous works~\citep{he2024mpgd} to formulate the per-timestep guidance as a constrained optimization problem. Given a conditional input $\bm{y}$, we formulate the latent refinement at each timestep as:
\begin{equation}~\label{eq:constraint_optim_prob}
    \bm{x}_{t-1} = \argmin_{\bm{x} \in \mathcal{S}_{t, r}} \mathcal{L}(\hat{\bm{x}}_0(\bm{x}, t-1), \bm{y}).
\end{equation}
Since we constrain the optimization respecting the feasible set defined in Equation~\ref{eq:spherical_manifold}, this optimization naturally respects the Gaussian properties of the diffusion process while improving the sample alignment toward the condition $\bm{y}$.

\begin{algorithm}[t]
    \caption{DiffRGD}~\label{algo:diffrgd}
    \begin{algorithmic}[1] 
        \Require{
        Reference image $\bm{y} \in \mathbb{R}^n$,
        Diffusion model $\bm{\epsilon}_\theta$,
        loss function $\mathcal{L}$,
        DDIM sampling schedule parameters $\{\bar{\alpha}_t\}_{t=1}^T$,
        noise level $\sigma_t > 0$,
        guidance strengths $\{\eta_t^{(k)}\}_{t=1}^T$,
        guidance inner iteration $K \in \mathbb{N}$.
        }
        \State{$\bm{x}_T \sim \mathcal{N}(\bm{0}, \bm{I}_n)$.}
        \For{$t=T$ to $1$}
            \State{$r \sim \chi(n)$.}
            \State{$\bm{u} \sim \operatorname{Unif}(\mathbb{S}^{n-1})$.}
            \State{$\bm{\mu}_t \leftarrow \sqrt{\bar{\alpha}_{t-1}} \hat{\bm{x}}_{0}(\bm{x}_t, t)+ \sqrt{1 - \bar{\alpha}_{t-1} - \sigma_t^2} \, \bm{\epsilon}_\theta(\bm{x}_t, t)$.}
            \State{$\bm{x}_{t-1}^{(0)} \leftarrow \bm{\mu}_t + \sigma_t r \bm{u}$.}
            \For{$k=0$ to $K-1$}
                \State{$\bm{g} \leftarrow \nabla_{\bm{x}_{t-1}^{(k)}} \mathcal{L}(\hat{\bm{x}}_{0}(\bm{x}_{t-1}^{(k)}, t-1), \bm{y})$.
                \Comment{Guidance gradient}
                }
                \State{$\grad_{\mathcal{S}_{t, r}} \mathcal{L} \leftarrow \mathrm{\Pi}_{\operatorname{T}_{\bm{x}_{t-1}^{(k)}}{\mathcal{S}_{t, r}}}(\bm{g})$.
                \Comment{Riemannian gradient}
                }
                \State{$\bm{x}_{t-1}^{(k+1)} \leftarrow \mathrm{R}_{\bm{x}_{t-1}^{(k)}}(-\eta_t^{(k)} \operatorname{grad}_{\mathcal{S}_{t, r}} \mathcal{L})$.
                \Comment{Retraction}
                }
            \EndFor
            \State{$\bm{x}_{t-1} \leftarrow \bm{x}_{t-1}^{(K)}$.}
        \EndFor
    \Ensure{Guided sample $\bm{x}_0$.}
    \end{algorithmic}
\end{algorithm}

\subsection{Riemannian Gradient Descent for Diffusion Guidance}
To solve the constrained optimization problem in  Equation~\ref{eq:constraint_optim_prob}, we employ Riemannian Gradient Descent (RGD) on the spherical manifold $\mathcal{S}_{t, r}$. We first define the geometric components for the optimization step as follows:

\paragraph{Tangent Space.}
The tangent space of the spherical manifold $\mathcal{S}_{t, r}$ at a point $\bm{x} \in \mathcal{S}_{t, r}$ consists of all directions orthogonal to the radial vector:
\begin{equation*}
    \operatorname{T}_{\bm{x}}{\mathcal{S}_{t, r}} := \left\{ \bm{v} \in \mathbb{R}^n \mid (\bm{x} - \bm{\mu}_t)^\top \bm{v} = 0 \right\}.
\end{equation*}

\paragraph{Riemannian Metric.}
We endow the manifold $\mathcal{S}_{t, r}$ with the standard Riemannian metric induced by the Euclidean inner product. For any two vectors in the tangent space $\bm{v}, \bm{w} \in \operatorname{T}_{\bm{x}}\mathcal{S}_{t, r}$, the inner product is defined as:
\begin{equation*}
    \langle \bm{v}, \bm{w} \rangle_{\bm{x}} := \bm{v}^\top \bm{w}.
\end{equation*}

\paragraph{Projection Operator.}
The projection of a vector $\bm{v} \in \mathbb{R}^n$ onto the tangent space $\operatorname{T}_{\bm{x}}{\mathcal{S}_{t, r}}$ is given by:
\begin{equation*}
    \mathrm{\Pi}_{\operatorname{T}_{\bm{x}}{\mathcal{S}_{t, r}}}(\bm{v}) := \left( \bm{I}_n - \frac{(\bm{x} - \bm{\mu}_t)(\bm{x} - \bm{\mu}_t)^\top}{\|\bm{x} - \bm{\mu}_t\|_2^2} \right) \bm{v},
\end{equation*}
This projection removes the radial part of $\bm{v}$ to ensure the update direction remains tangent to the spherical manifold.

\paragraph{Retraction.}
After taking a step along the tangent space, the resulting latent may leave the spherical manifold. We apply a retraction operation to map the updated latent back to the spherical manifold by:
\begin{equation*}
    \operatorname{R}_{\bm{x}}(\bm{v}) := \bm{\mu}_t + \sigma_t r \cdot \frac{\bm{x} - \bm{\mu}_t + \bm{v}}{\|\bm{x} - \bm{\mu}_t + \bm{v}\|_2},
\end{equation*}
which ensures the latent updates lying on the same spherical manifold specified by the radius $\sigma_t r$.

\begin{wrapfigure}{r}{0.45\linewidth}
    \centering
    \vspace{-10pt}
    \includegraphics[width=\linewidth]{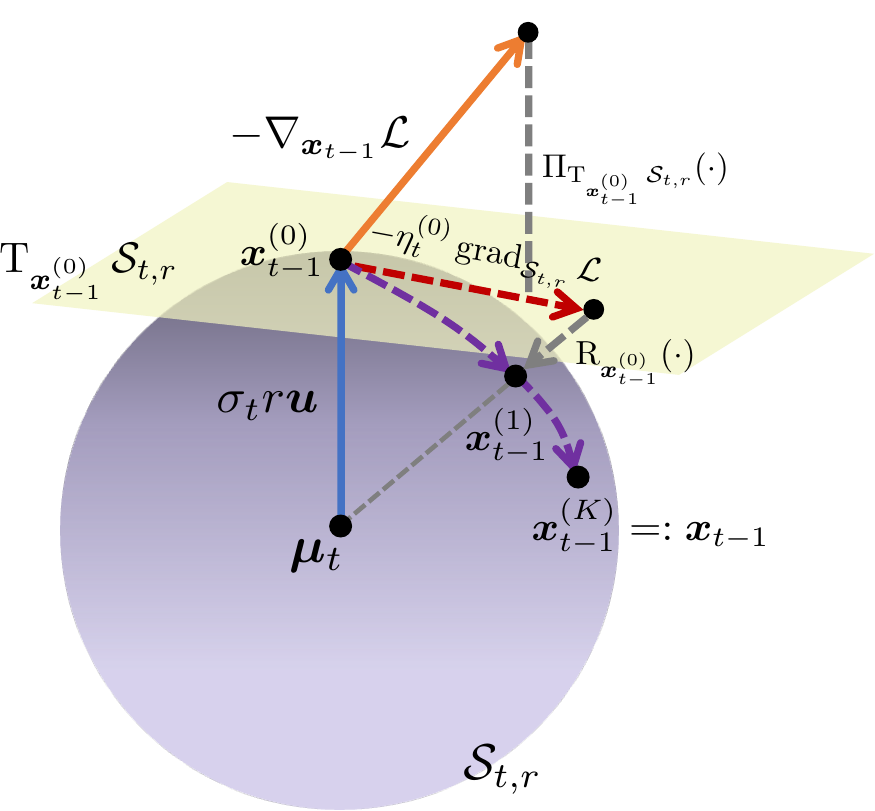}
    \caption{An illustration of DiffRGD.
    }
    \label{fig:rgd}
    \vspace{-10pt}
\end{wrapfigure}

Fig.~\ref{fig:rgd} shows the guidance step in DiffRGD. For each diffusion timestep $t$, given the predicted mean $\bm{\mu}_t$ by DDIM and the radius $\sigma_t r$, the latent $\bm{x}_{t-1}^{(0)}$ lies on the constructed spherical manifold $\mathcal{S}_{t, r}$ before applying guidance. Starting from this initial latent, we perform $K$ Riemannian optimization steps to refine the latent by the guiding loss while ensuring it lies on the constructed manifold.

In each Riemannian optimization iteration, we compute the negative Euclidean gradient $-\nabla_{\bm{x}_{t-1}^{(k)}} \mathcal{L}$, where 
$\mathcal{L}(\cdot, \bm{y})$ denotes the guidance objective conditioned on $\bm{y}$. Then, we use this gradient to obtain the Riemannian gradient and perform retraction for the guidance update. Specifically, to respect the spherical constraint, this gradient is first projected onto the tangent space $\operatorname{T}_{\bm{x}_{t-1}^{(k)}} \mathcal{S}_{t, r}$, yielding the Riemannian gradient:
\begin{equation*}
    \operatorname{grad}_{\mathcal{S}_{t, r}} \mathcal{L}
    :=
    \mathrm{\Pi}_{\operatorname{T}_{\bm{x}_{t-1}^{(k)}}{\mathcal{S}_{t, r}}}
    \left(
        \nabla_{\bm{x}_{t-1}^{(k)}} 
        \mathcal{L}(\hat{\bm{x}}_0(\bm{x}_{t-1}^{(k)}, t-1), \bm{y})
    \right).
\end{equation*}
We then update the latent along the negative Riemannian gradient direction with guidance strength $\eta_t^{(k)} > 0$ via a retraction-based step:
\begin{equation*}
    \bm{x}_{t-1}^{(k+1)}
    =
    \operatorname{R}_{\bm{x}_{t-1}^{(k)}}
    \left(
        -\eta_{t}^{(k)} 
        \operatorname{grad}_{\mathcal{S}_{t, r}} \mathcal{L}
    \right).
\end{equation*}
The retraction ensures that each update remains on the spherical manifold $\mathcal{S}_{t, r}$, preventing the radial deviation that naive Euclidean gradient updates would otherwise introduce. After $K$ iterations, we obtain the optimized latent 
$\bm{x}_{t-1}^{(K)}$, which serves as the final guided sample $\bm{x}_{t-1}$.

The full algorithm is shown in Algorithm~\ref{algo:diffrgd}. This framework ensures each optimization step remains on the spherical manifold induced by the isotropic Gaussian, thereby preserving the validity of the sampling process and mitigating quality degradation caused by off-manifold phenomena. A detailed analysis of the distribution-preserving property is provided in the Appendix~\ref{appendix:distribution}.

\subsection{Convergence Analysis}
To ensure the stability and reliability of our guidance framework, it is essential to study the convergence behavior of the optimization steps. In this section, we theoretically justify the convergence of our optimization loop. In Theorem~\ref{theorem:convergence}, we show that at each diffusion sampling step, our method asymptotically converges to a stationary point.

\begin{theorem}[Convergence to Stationary Point of DiffRGD] \label{theorem:convergence}
Let $\bm{f}_t : \mathcal{S}_{t, r} \to \mathbb{R}$ be the guidance objective at timestep $t$, defined as $\bm{f}_t(\bm{x}) := \mathcal{L}(\hat{\bm{x}}_0(\bm{x}, t-1), \bm{y})$, and let $\{ \bm{x}_{t-1}^{(k)} \}$ denote the sequence produced by Algorithm 1 at timestep $t$. Assume that $\bm{f}_t$ is $L_t$-smooth on the spherical manifold $\mathcal{S}_{t, r}$, and the guidance strength $0< \eta_\mathrm{min} \leq \eta_{t}^{(k)} \leq \frac{1}{L_t}$. Then the norm of the Riemannian gradient satisfies
\begin{equation*}
    \left\|\operatorname{grad}_{\mathcal{S}_{t, r}} \bm{f}_t(\bm{x}_{t-1}^{(k)}) \right\|_2 \to 0 \text{ as } k \to \infty.
\end{equation*}
with a convergence rate $\mathcal{O}(\frac{1}{\sqrt{k}})$.
\end{theorem}

This theorem shows that, under the assumption of $L$-smoothness, Algorithm~\ref{algo:diffrgd} converges to a stationary point at each timestep with a sublinear rate. In our task-specific application, we observe that the Riemannian gradient norm typically becomes small within three iterations. Therefore, in practice, we fix the number of optimization steps to $K = 3$, and provide the convergence curves of the gradient norm across several timesteps in the Appendix~\ref{appendix:inner_loop}.

\section{Experiment Results}
For a fair comparison, we reimplement several diffusion guidance baselines under identical data and preprocessing, including DPS~\citep{chung2023dps}, FreeDoM~\citep{yu2023freedom}, MPGD~\citep{he2024mpgd}, DSG~\citep{yang2024dsg}, and ADMMDiff~\citep{zhang2025admmdiff}.
Detailed algorithmic descriptions of each method, along with the full set of hyperparameters used in our experiments, are provided in the Appendix~\ref{appendix:algos} and Appendix~\ref{appendix:params}. We follow the official configurations reported in the original papers wherever available. For methods lacking complete parameter specifications in either the paper or public codebase, we report the best-performing hyperparameter settings. All experiments used the DDIM sampling schedule with $\sigma_t = \sqrt{(1-\bar{\alpha}_{t-1}) / (1-\bar{\alpha}_{t})} \cdot \sqrt{1-\bar{\alpha}_t / \bar{\alpha}_{t-1}}$ and were implemented in PyTorch~\citep{paszke2019pytorch} with Hugging Face Diffusers~\citep{platen2022diffusers} and ran on a single NVIDIA GeForce RTX 4090 GPU. In the following sections, we evaluate baseline methods and our proposed approach on image restoration tasks (Section~\ref{sec:image_restoration}) and conditional generation tasks (Section~\ref{sec:condition_gen}).

\subsection{Image Restoration} \label{sec:image_restoration}
We evaluate our proposed method alongside representative baseline approaches on four image restoration tasks. The objective in each task is to reconstruct a clean image $\bm{x}$ from a degraded observation $\bm{y}$. Each task is formulated as a noisy linear inverse problem~\citep{tikhonov1977inverseproblem, chung2023dps} modeled by the forward process:
\begin{equation*}
    \bm{y} = \mathcal{A}(\bm{x}) + \bm{n}, \,
    \bm{n} \sim \mathcal{N}(\bm{0}, 0.05^2 \bm{I})
\end{equation*}
where $\mathcal{A}(\cdot)$ denotes a task-specific linear degradation operator. Given the noisy measurement $\bm{y}$, the goal is to reconstruct $\hat{\bm{x}}_0$ conditioned on a diffusion prior such that $\mathcal{A}(\hat{\bm{x}}_0)$ remains consistent with $\bm{y}$. To enable conditional generation under each inverse setting, we define the guidance objective using an $\ell_2$ reconstruction loss:
\begin{equation*}
    \mathcal{L}_\text{inv}(\hat{\bm{x}}_0, \bm{y}) = \| \mathcal{A}(\hat{\bm{x}}_0) - \bm{y} \|_2.
\end{equation*}

\begin{table}[h]
    \setlength{\tabcolsep}{6pt}
    \centering
    \caption{Quantitative results using DDIM 1,000 sampling steps on 150 samples from the FFHQ 256 $\times$ 256 validation set for image restoration tasks. Metrics are averaged over 100 bootstrap runs. * indicates a statistically significant improvement ($p < 0.05$) over the second-best method.
    }
    \label{tab:ffhq_ddim1000}
    \resizebox{0.95\linewidth}{!}{
    \begin{tabular}{llccccc}
        \toprule
        \textbf{Task} & \textbf{Method} & \textbf{Venue} & \textbf{PSNR $\uparrow$} & \textbf{SSIM $\uparrow$} & \textbf{LPIPS $\downarrow$} & \textbf{FID $\downarrow$} \\
        \midrule
        \multirow{5}{*}{Inpainting}
        & DPS~\citep{chung2023dps} & ICLR 2023 & $30.44$ & $0.863$ & $0.153$ & $42.68$ \\
        & MPGD~\citep{he2024mpgd} & ICLR 2024 & $27.51$ & $0.724$ & $0.256$ & $68.24$ \\
        & DSG~\citep{yang2024dsg} & ICML 2024 & $31.03$ & $0.866$ & $0.144$ & $36.30$ \\
        & ADMMDiff~\citep{zhang2025admmdiff} & CVPR 2025 & $\second{32.38}$ & $\second{0.899}$ & $\second{0.119}$ & $\second{29.52}$ \\
        & \cellcolor{blue!10}DiffRGD (Ours) & \cellcolor{blue!10}ECCV 2026 & \cellcolor{blue!10}$\first{34.04}^{*}$ & \cellcolor{blue!10}$\first{0.926}^{*}$ & \cellcolor{blue!10}$\first{0.096}^{*}$ & \cellcolor{blue!10}$\first{21.88}^{*}$ \\
        \midrule
        \multirow{5}{*}{Super-Resolution $4 \times$}
        & DPS~\citep{chung2023dps} & ICLR 2023 & $26.03$ & $0.727$ & $0.260$ & $80.05$ \\
        & MPGD~\citep{he2024mpgd}  & ICLR 2024 & $24.40$ & $0.614$ & $0.354$ & $101.50$ \\
        & DSG~\citep{yang2024dsg}  & ICML 2024 & $\second{26.71}$ & $\second{0.737}$ & $\second{0.256}$ & $\second{74.67}$ \\
        & ADMMDiff~\citep{zhang2025admmdiff} & CVPR 2025 & $26.48$ & $0.712$ & $0.297$ & $96.69$ \\
        & \cellcolor{blue!10}DiffRGD (Ours) & \cellcolor{blue!10}ECCV 2026 & \cellcolor{blue!10}$\first{27.77}^{*}$ & \cellcolor{blue!10}$\first{0.783}^{*}$ & \cellcolor{blue!10}$\first{0.220}^{*}$ & \cellcolor{blue!10}$\first{63.94}^{*}$ \\
        \midrule
        \multirow{5}{*}{Gaussian Deblurring}
        & DPS~\citep{chung2023dps} & ICLR 2023 & $25.88$ & $0.721$ & $0.237$ & $69.38$ \\
        & MPGD~\citep{he2024mpgd} & ICLR 2024 & $24.07$ & $0.576$ & $0.328$ & $95.12$ \\
        & DSG~\citep{yang2024dsg}  & ICML 2024 & $\first{27.45}$ & $0.751$ & $0.259$ & $\second{75.85}$ \\
        & ADMMDiff~\citep{zhang2025admmdiff} & CVPR 2025 & $26.57$ & $\first{0.757}$ & $\second{0.226}$ & $79.30$ \\
        & \cellcolor{blue!10}DiffRGD (Ours) & \cellcolor{blue!10}ECCV 2026 & \cellcolor{blue!10}$\second{26.80}$ & \cellcolor{blue!10}$\first{0.757}$ & \cellcolor{blue!10}$\first{0.218}^{*}$ & \cellcolor{blue!10}$\first{63.83}^{*}$ \\
        \midrule
        \multirow{5}{*}{Motion Deblurring}
        & DPS~\citep{chung2023dps} & ICLR 2023 & $24.47$ & $0.685$ & $0.271$ & $80.75$ \\
        & MPGD~\citep{he2024mpgd} & ICLR 2024 & $23.15$ & $0.569$ & $0.357$ & $106.99$ \\
        & DSG~\citep{yang2024dsg}  & ICML 2024 & $\second{26.80}$ & $0.709$ & $0.290$ & $87.96$ \\
        & ADMMDiff~\citep{zhang2025admmdiff} & CVPR 2025 & $\first{27.26}$ & $\first{0.778}$ & $\first{0.222}$ & $72.92$ \\
        & \cellcolor{blue!10}DiffRGD (Ours) & \cellcolor{blue!10}ECCV 2026 & \cellcolor{blue!10}$25.84$ & \cellcolor{blue!10}$\second{0.736}$ & \cellcolor{blue!10}$\second{0.250}$ & \cellcolor{blue!10}$\first{72.26}$  \\
        \bottomrule
    \end{tabular}
    }
\end{table}

\begin{figure}[h]
    \centering
    \includegraphics[width=\linewidth]{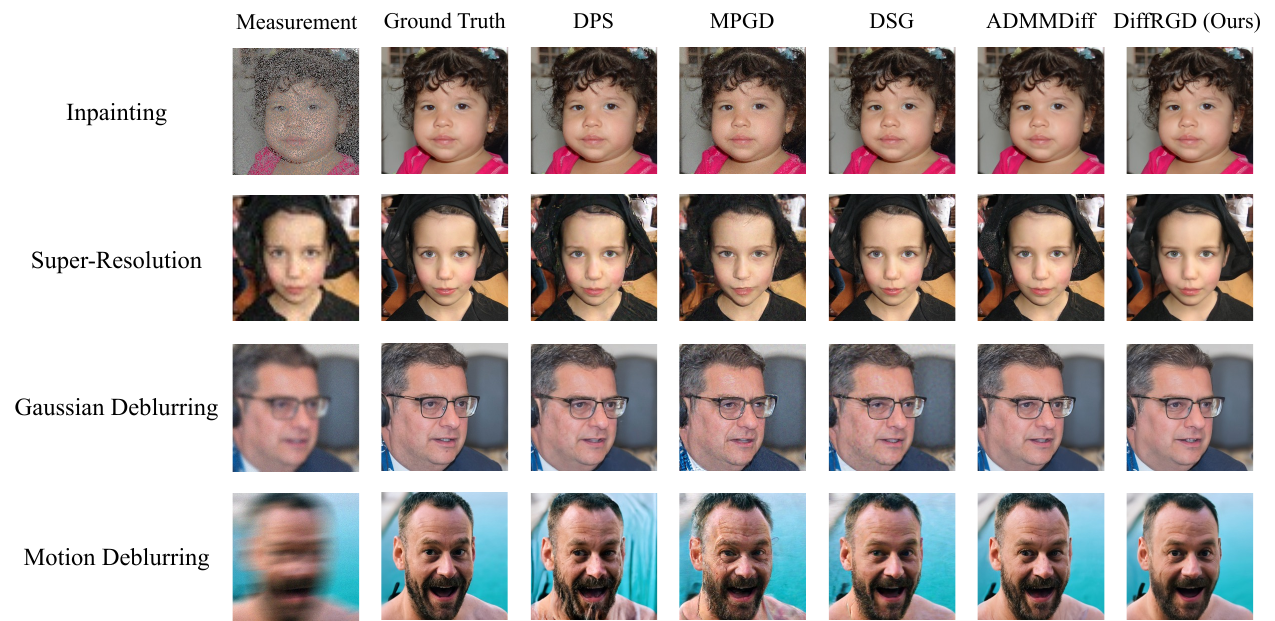}
    \caption{Qualitative comparison for image restoration tasks on FFHQ 256 $\times$ 256
    validation set. Our proposed method achieves better restoration results with fewer artifacts than other compared methods.
    }
    \label{fig:exp_ffhq}
\end{figure}

\clearpage
\begin{table*}[t]
    \setlength{\tabcolsep}{6pt}
    \centering
    \caption{
    Quantitative results using DDIM 100 sampling steps on 150 samples from the FFHQ 256 $\times$ 256 validation set for deblurring tasks. Metrics are averaged over 100 bootstrap runs. * indicates a statistically significant improvement ($p < 0.05$) over the second-best method.
    }
    \label{tab:ffhq_ddim100_1}
    \resizebox{0.9\linewidth}{!}{
    \begin{tabular}{lcccccc}
        \toprule
        \multirow{2.5}{*}{\textbf{Method}} & \multicolumn{3}{c}{\textbf{Gaussian Deblurring}} & \multicolumn{3}{c}{\textbf{Motion Deblurring}} \\
        \cmidrule(lr){2-4} \cmidrule(lr){5-7}
        ~ & \textbf{PSNR $\uparrow$} & \textbf{SSIM $\uparrow$} & \textbf{LPIPS $\downarrow$} & \textbf{PSNR $\uparrow$} & \textbf{SSIM $\uparrow$} & \textbf{LPIPS $\downarrow$} \\
        \midrule
        DPS~\citep{chung2023dps} & $24.28$ & $0.680$ & $0.302$ & $20.63$ & $0.580$ & $0.372$ \\
        MPGD~\citep{he2024mpgd} & $24.60$ & $0.610$ & $0.312$ & $23.13$ & $0.579$ & $0.356$ \\    
        DSG~\citep{yang2024dsg} & $\second{27.02}$ & $\second{0.758}$ & $\second{0.261}$ & $\second{24.81}$ & $\second{0.705}$ & $\first{0.290}$ \\
        ADMMDiff~\citep{zhang2025admmdiff} & $26.01$ & $0.734$ & $0.278$ & $23.16$ & $0.664$ & $0.327$ \\
        \rowcolor{blue!10}
        DiffRGD (Ours) & $\first{27.46}^{*}$ & $\first{0.781}^{*}$ & $\first{0.235}^{*}$ & $\first{24.88}$ & $\first{0.717}$ & $\second{0.294}$ \\
        \bottomrule
    \end{tabular}
    }
\end{table*}

\begin{table*}[h]
    \setlength{\tabcolsep}{6pt}
    \caption{
    Quantitative results using DDIM 100 sampling steps on 150 samples from the ImageNet 256 $\times$ 256 validation set for inpainting and super-resolution. Metrics are averaged over 100 bootstrap runs. * indicates a statistically significant improvement ($p < 0.05$) over the second-best method.
    }
    \centering
    \resizebox{0.9\linewidth}{!}{
    \begin{tabular}{lcccccc}
        \toprule
        \multirow{2.5}{*}{\textbf{Method}} & \multicolumn{3}{c}{\textbf{Inpainting}} & \multicolumn{3}{c}{\textbf{Super-Resolution} $4 \times$} \\
        \cmidrule(lr){2-4} \cmidrule(lr){5-7}
        ~ & \textbf{PSNR $\uparrow$} & \textbf{SSIM $\uparrow$} & \textbf{LPIPS $\downarrow$} & \textbf{PSNR $\uparrow$} & \textbf{SSIM $\uparrow$} & \textbf{LPIPS $\downarrow$} \\
        \midrule
        DPS~\citep{chung2023dps} & $23.24$ & $0.633$ & $0.349$ & $20.26$ & $0.500 $ & $0.449$ \\
        MPGD~\citep{he2024mpgd} & $26.15$ & $0.711$ & $0.323$ & $21.50$ & $0.531$ & $0.478$ \\    
        DSG~\citep{yang2024dsg} & $26.58$ & $\second{0.754}$ & $\second{0.255}$ & $22.82$ & $0.585$ & $\second{0.415}$ \\
        ADMMDiff~\citep{zhang2025admmdiff} & $\second{27.97}$ & $0.744$ & $0.303$ & $\second{24.45}$ & $\second{0.633}$ & $0.430$ \\
        \rowcolor{blue!10}
        DiffRGD (Ours) & $\first{31.16}^{*}$ & $\first{0.875}^{*}$ & $\first{0.208}^{*}$ & $\first{24.95}^{*}$ & $\first{0.663}^{*}$ & $\first{0.379}^{*}$ \\
        \bottomrule
    \end{tabular}
    }
    \label{tab:imagenet_ddim100_1}
    \vspace{-10pt}
\end{table*}

\paragraph{Implementation Details.}
We consider four inverse problem tasks using 150 images from FFHQ 256 $\times$ 256~\citep{karras2019ffhq} validation datasets. We utilize pre-trained diffusion models provided by \citep{chung2023dps, dhariwal2021cg} as the priors for all guidance frameworks. Each inverse problem is defined by a different degradation operator $\mathcal{A}$: (i) \textbf{Inpainting:} A random binary mask is applied to simulate missing pixels, where each pixel has a $70\%$ chance of being masked. (ii) \textbf{Super-Resolution:} Input images are downsampled by a factor of $4$ using bicubic interpolation. (iii) \textbf{Gaussian Deblurring:} A Gaussian blur kernel of size $31 \times 31$ with a standard deviation $3.0$ is applied to simulate blur degradation. (iv) \textbf{Motion Deblurring:} A motion blur kernel of size $61 \times 61$ is applied with an intensity of $0.5$, and the direction of motion is randomly sampled.

\paragraph{Evaluation Protocol.}
We evaluate performance using PSNR~\citep{jain1989dip}, SSIM~\citep{wang2004ssim}, LPIPS~\citep{zhang2018lpips}, and FID~\citep{heusel2017fid} (details are provided in the Appendix~\ref{appendix:metrics}).
To increase the robustness of the metrics and assess statistical significance, we perform bootstrap resampling~\citep{tibshirani1993bootstrap}. For each method, we report the mean performance across all metrics over 100 bootstrap runs, where 100 samples are randomly drawn from the 150-sample evaluation set in each run.

\paragraph{Quantitative Results.}
Table~\ref{tab:ffhq_ddim1000} reports quantitative results using 1,000 DDIM sampling steps, where our method outperforms other baselines on inpainting, super-resolution, and Gaussian deblurring. To assess the robustness of our method with fewer denoising timesteps, we conduct experiments using 100 steps and report in Table~\ref{tab:ffhq_ddim100_1}. As the number of timesteps decreases, DPS and ADMMDiff exhibit noticeable performance degradation, whereas DiffRGD maintains strong performance. We further evaluate inpainting and super-resolution on ImageNet $256 \times 256$~\citep{deng2009imagenet} (Table~\ref{tab:imagenet_ddim100_1}), where DiffRGD continues to outperform the baselines. Additional quantitative results are provided in the Appendix~\ref{appendix:exp}.

\paragraph{Qualitative Results.}
Figure~\ref{fig:exp_ffhq} shows the qualitative results across four image restoration tasks. We can observe that DPS and MPGD introduce some perceptible artifacts in the background across four tasks; DSG performs well on inpainting and super-resolution tasks, but introduces some artifacts in both deblurring tasks. DiffRGD achieves comparable visual quality with ADMMDiff, but with slightly better artifact control and demands less time.

\subsection{Super-Resolution under Various Subsampling Rates}
\label{appendix:super_res}

We further evaluate our method using the Stable Diffusion model~\citep{rombach2022ldm} under various subsampling rates. Specifically, the original FFHQ images at a resolution of $1024 \times 1024$ are downsampled to $64 \times 64$ and then super-resolved to $512 \times 512$ ($8 \times$) and $768 \times 768$ ($12 \times$). As reported in Table~\ref{tab:super_res_ablation}, DiffRGD consistently outperforms both DPS and DSG across both upsampling scales in terms of quantitative metrics. Qualitative results are presented in Figure~\ref{fig:super_res_ablation}, which shows that higher resolutions yield sharper image details. However, the $12\times$ outputs show greater discrepancy in fine-grained structures compared to $8\times$, as reflected in Table 4: while the $12 \times$ results achieve higher SSIM, they exhibit lower PSNR.

\begin{figure*}[h]
    \centering
    \includegraphics[width=0.65\linewidth]{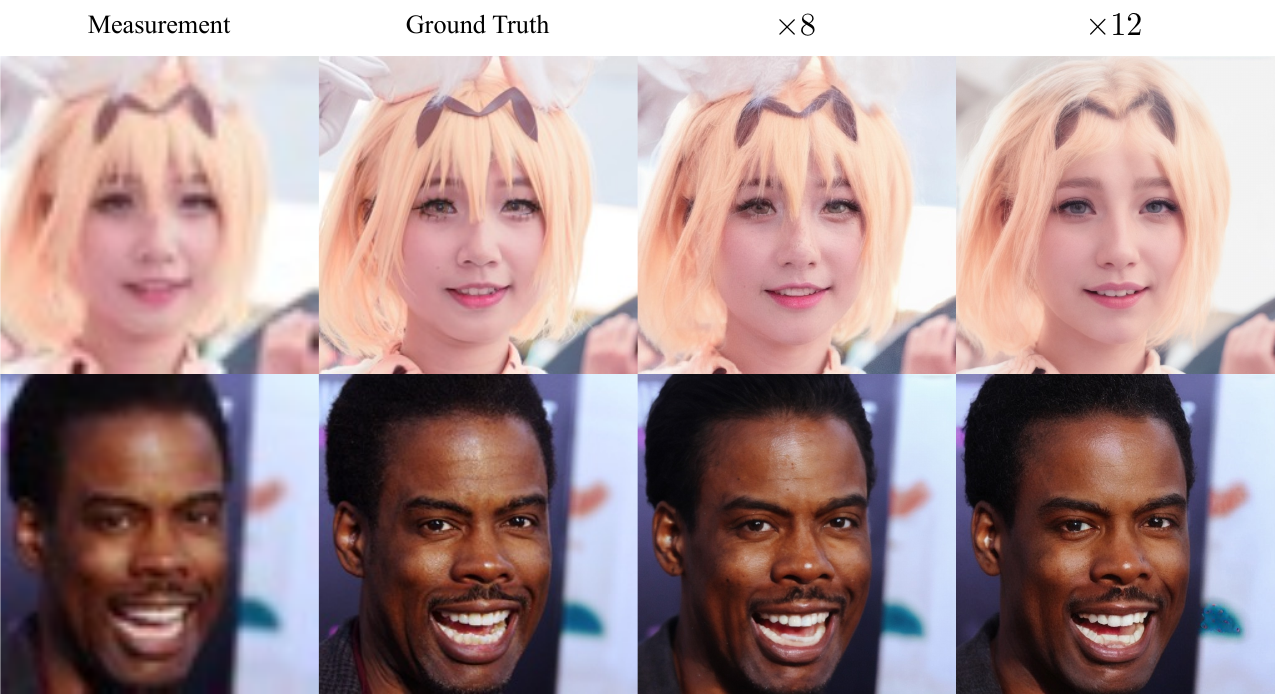}
    \caption{Qualitative results of super-resolution by our DiffRGD method with Stable Diffusion under $\times 8$ and $\times 12$ subsampling rates.}
    \label{fig:super_res_ablation}
    \vspace{-5pt}
\end{figure*}

\begin{table}[h]
    \centering
    \caption{Quantitative results of $8\times$ and $12\times$ super-resolution with Stable Diffusion.}
    \setlength{\tabcolsep}{4pt}
    \begin{tabular}{lcccc}
        \toprule
        \multirow{2.5}{*}{\textbf{Method}} & \multicolumn{2}{c}{$\times 8 \, (64 \rightarrow 512)$} & \multicolumn{2}{c}{$\times 12 \, (64 \rightarrow 768)$} \\
        \cmidrule(lr){2-3} \cmidrule(lr){4-5}
       ~ & \textbf{PSNR $\uparrow$} & \textbf{SSIM $\uparrow$} & \textbf{PSNR $\uparrow$} & \textbf{SSIM $\uparrow$} \\
        \midrule
        DPS~\citep{chung2023dps} & $25.67$ & $0.689$ & $24.99$ & $0.689$ \\
        DSG~\citep{yang2024dsg}  & $25.38$ & $0.679$ & $26.02$ & $\first{0.709}$ \\
        \rowcolor{blue!10}
        DiffRGD (Ours) & $\first{26.41}$ & $\first{0.700}$ & $\first{26.39}$ & $\first{0.709}$ \\
        \bottomrule
    \end{tabular}
    \label{tab:super_res_ablation}
    \vspace{-10pt}
\end{table}

\clearpage
\begin{table*}[t]
    \setlength{\tabcolsep}{2pt}
    \centering
    \caption{
    Quantitative results using DDIM 100 sampling steps on 150 samples from the CelebA-HQ 256 $\times$ 256 validation set for conditional image generation tasks. Metrics are averaged over 100 bootstrap runs. * indicates a statistically significant improvement ($p < 0.05$) over the second-best method.
    }
    \label{tab:condition_gen}
    \resizebox{\linewidth}{!}{
    \begin{tabular}{lccccccccc}
        \toprule
        \multirow{2.5}{*}{\textbf{Method}} & \multicolumn{3}{c}{\textbf{Segmentation Map}} & \multicolumn{3}{c}{\textbf{Sketch}} & \multicolumn{3}{c}{\textbf{FaceID}} \\
        \cmidrule(lr){2-4} \cmidrule(lr){5-7} \cmidrule(lr){8-10}
        & ~ \textbf{mIoU $\uparrow$} & \textbf{FID $\downarrow$} & \textbf{KID $\downarrow$} & \textbf{Sketch-$\ell_2$ $\downarrow$} & \textbf{FID $\downarrow$} & \textbf{KID $\downarrow$} & \textbf{FaceID-$\ell_2$ $\downarrow$} & \textbf{FID $\downarrow$} & \textbf{KID $\downarrow$}  \\
        \midrule
        FreeDoM~\citep{yu2023freedom} & $0.622$ & $156.02$ & $0.057$ & $30.85$ & $101.90$ & $0.017$ & $0.557$ & $127.05$ & $0.032$ \\
        DSG~\citep{yang2024dsg} & $0.750$ & $117.48$ & $\second{0.028}$ & $\second{21.36}$ & $107.00$ & $0.024$ & $\second{0.340}$ & $\second{95.27}$ & $\second{0.014}$ \\
        ADMMDiff~\citep{zhang2025admmdiff} & $\second{0.758}$ & $\second{101.86}$ & $0.031$ & $30.82$ & $\second{97.52}$ & $\second{0.014}$ & $0.346$ & $100.81$ & $\second{0.014}$ \\
        \rowcolor{blue!10}
        DiffRGD (Ours) & $\first{0.804}^{*}$ & $\first{96.10}^{*}$ & $\first{0.026}$ & $\first{19.48}^{*}$ & $\first{87.82}^{*}$ & $\first{0.012}$ & $\first{0.303}^{*}$ & $\first{93.80}$ & $\first{0.011}$ \\
        \bottomrule
    \end{tabular}
    }
\end{table*}

\begin{figure}[t]
    \centering
    \includegraphics[width=\linewidth]{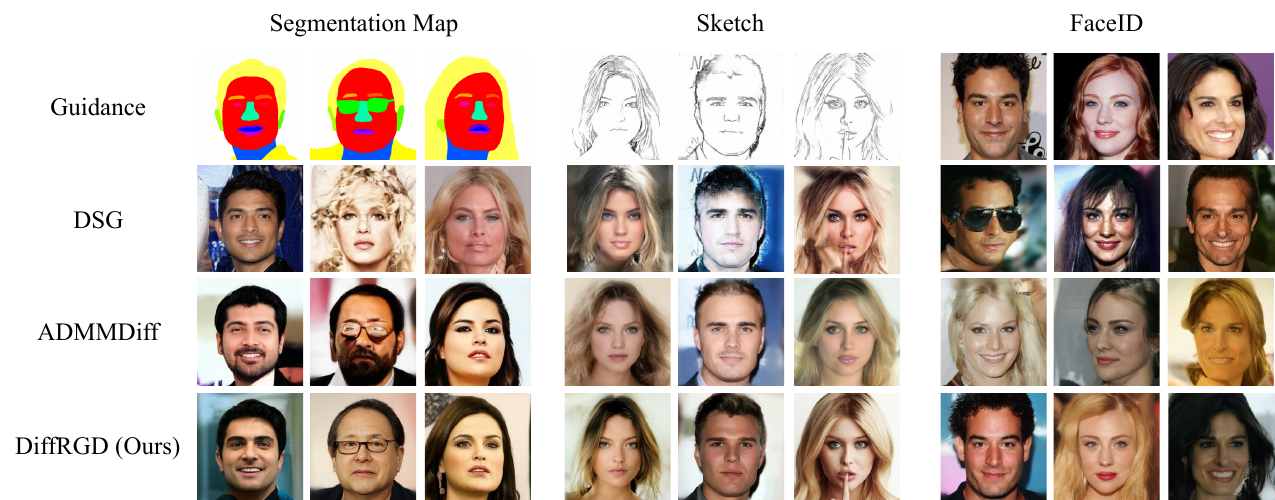}
    \caption{Qualitative results for conditional generation tasks on CelebA-HQ 256 $\times$ 256 validation set. (a) segmentation maps (b) sketches (c) FaceID. Our method generates results that better align with the given conditions than other state-of-the-art methods.}
    \label{fig:exp_condition}
\end{figure}

\subsection{Conditional Generation} \label{sec:condition_gen}
We further evaluate our method against representative
baselines on three conditional generation tasks: segmentation maps, sketches, and FaceID guidance for human face generation. We also provide style-guided generation results in the Appendix~\ref{appendix:style_guide}.

Given a conditional image $\bm{y}$, we extract task-specific features using a pre-trained model $\psi_\theta$. The conditional guidance loss is then defined as:
\begin{equation*}
    \mathcal{L}_\text{cond}(\hat{\bm{x}}_0, \bm{y}) = \| \psi_\theta(\hat{\bm{x}}_0) - \psi_\theta(\bm{y}) \|_2.
\end{equation*}

\paragraph{Implementation Details.}
We adopt pre-trained diffusion models provided by \citep{mengsdedit} as the generative backbone. The three tasks are defined as follows:
(i) \textbf{Segmentation Map Guidance:} We use a pre-trained BiSeNet~\citep{yu2018bisenet} to extract face parsing maps from both generated and reference images. (ii) \textbf{Sketch Guidance:} A pre-trained AODA model~\citep{xiang2022aoda} is used to parse structural cues from input sketches. (iii) \textbf{FaceID Guidance:} We utilize a pre-trained ArcFace model~\citep{deng2019arcface} to extract identity embeddings for evaluating face similarity. 

\paragraph{Evaluation Protocol.}
We evaluate our method on three conditional generation tasks using 150 validation images from the CelebA-HQ $256 \times 256$~\citep{karras2018celeba} dataset. To increase the robustness of the metrics, we perform bootstrap resampling~\citep{tibshirani1993bootstrap}. For each method, we report the mean performance across all metrics over 100 bootstrap runs. Details of the evaluation metrics are provided in the Appendix~\ref{appendix:metrics}.

\paragraph{Quantitative Results.}
Table~\ref{tab:condition_gen} presents the quantitative results, showing that our method consistently outperforms existing state-of-the-art inference-time methods. The generated samples exhibit improved visual fidelity and demonstrate stronger alignments with the given conditional inputs.

\paragraph{Qualitative Results.}
Figure~\ref{fig:exp_condition} presents the qualitative results. DSG shows unstable control and often produces noticeable artifacts. DiffRGD and ADMMDiff generate samples that align well with the guidance, while DiffRGD yields more natural and visually appealing results under the same time constraint.

\subsection{Influence of Guidance Strength}
\label{sec:guidance_strength}

We investigate the influence of guidance strength $\eta_t$ on DSG and DiffRGD. Figure~\ref{fig:exp_different_guidance_strength} shows qualitative results under segmentation map guidance. DiffRGD better aligns with the conditioning signals while preserving natural visual quality across a wide range of guidance strengths. In contrast, DSG shows poor controllability at small guidance strengths and severe degradation at large guidance strengths. Table~\ref{tab:ablation_guidance_rate} reports quantitative results on 100 CelebA-HQ $256 \times 256$ samples, further showing that DiffRGD achieves a better mIoU--FID trade-off than DSG and remains stable across guidance strengths.

\begin{figure}[h]
    \centering
    \includegraphics[width=0.7\linewidth]{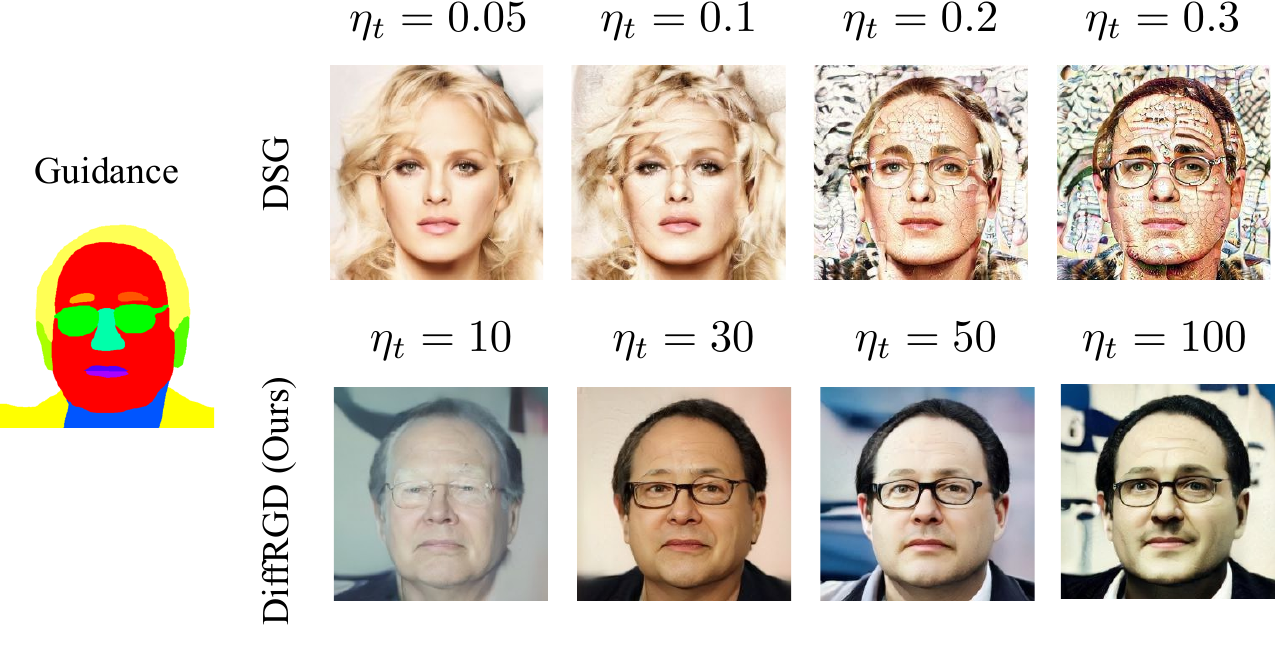}
    \caption{Comparison of different guidance strengths for segmentation map-guided conditional generation.}
    \label{fig:exp_different_guidance_strength}
    \vspace{-10pt}
\end{figure}

\begin{table}[h]
    \centering
    \caption{Influence of guidance strength $\eta_t$ on segmentation map conditional generation with DSG and DiffRGD.}
    \setlength{\tabcolsep}{4pt}
    \begin{tabular}{lcccccccc}
        \toprule
        ~ & \multicolumn{4}{c}{DSG} & \multicolumn{4}{c}{DiffRGD} \\
        \cmidrule(lr){2-5} \cmidrule(lr){6-9}
        $\eta_t$ & 0.05 & 0.1 & 0.2 & 0.3 & 10 & 30 & 50 & 100 \\
        \midrule
        mIoU $\uparrow$ & $0.53$ & $0.76$ & $0.90$ & $0.93$ & $0.74$ & $0.82$ & $0.84$ & $0.86$ \\
        FID $\downarrow$ & $95.69$ & $110.00$ & $157.82$ & $248.77$ & $80.73$ & $89.03$ & $93.47$ & $108.33$ \\
        \bottomrule
    \end{tabular}
    \label{tab:ablation_guidance_rate}
\end{table}

\clearpage

\section{Conclusion}
To address the sample quality degradation arising from the failure to respect the stepwise Gaussian latent structure of diffusion models, we propose a novel inference-time guidance method, \emph{DiffRGD}. Our method constrains updates to a latent-distribution-induced spherical manifold and solves each step via Riemannian Gradient Descent, maintaining distributional fidelity while enabling effective conditioning. On image restoration and conditional generation tasks, DiffRGD matched or surpassed prior inference-time methods, improving task metrics while preserving sample quality. These results indicate that geometry-consistent guidance is a viable alternative to directly applying the gradient in the ambient space or in a restrictive proposed manifold. In sum, respecting diffusion’s latent distribution yields a simple, plug-and-play guidance with strong practical payoff.

\section*{Acknowledgements}
This research is supported by National Science and Technology Council, Taiwan (R.O.C.), under the grant number of NSTC-113-2115-M-003-012-MY2, NSTC-113-2221-E-002-201, NSTC-113-2221-E-007-105-MY3, NSTC-114-2221-E-001-004, NSTC-114-2221-E-002-182-MY3, NSTC-114-2634-F-001-001-MBK, NSTC-114-2634-F-002-004, and Academia Sinica under the grant number of AS-IAIA-114-M10. We sincerely thank Kang-Yang Huang for the discussions and valuable feedback during the rebuttal stage. We also thank Yu-Hsuan Lin and Hsien-Yueh Huang for their continuous encouragement throughout this project.

\clearpage
{
    \bibliographystyle{iclr2026_conference}
    \bibliography{ref}
}

\clearpage
\appendix
\renewcommand{\thesection}{\Alph{section}}
\renewcommand{\theHsection}{\Alph{section}}
\setcounter{section}{0}
\begin{center}
    \LARGE{\textbf{Appendix}}
\end{center}
\begin{itemize}[leftmargin=0pt]
    \item[] \hyperref[appendix:notation]{\ref*{appendix:notation}.~\nameref*{appendix:notation}} \dotfill \pageref{appendix:notation}

    \item[] \hyperref[appendix:math]{\ref*{appendix:math}.~\nameref*{appendix:math}} \dotfill \pageref{appendix:math}

    \item[] \hyperref[appendix:implement]{\ref*{appendix:implement}.~\nameref*{appendix:implement}} \dotfill \pageref{appendix:implement}

    \begin{itemize}
        \item[] \hyperref[appendix:algos]{\ref*{appendix:algos}.~\nameref*{appendix:algos}} \dotfill \pageref{appendix:algos}
        \item[] \hyperref[appendix:params]{\ref*{appendix:params}.~\nameref*{appendix:params}} \dotfill \pageref{appendix:params}
        \item[] \hyperref[appendix:metrics]{\ref*{appendix:metrics}.~\nameref{appendix:metrics}} \dotfill \pageref{appendix:metrics}
    \end{itemize}

    \item[] \hyperref[appendix:analysis]{\ref*{appendix:analysis}.~\nameref*{appendix:analysis}} \dotfill \pageref{appendix:analysis}
    \begin{itemize}
        \item[] \hyperref[appendix:distribution]{\ref*{appendix:distribution}.~\nameref*{appendix:distribution}} \dotfill \pageref{appendix:distribution}
        \item[] \hyperref[appendix:inner_loop]{\ref*{appendix:inner_loop}.~\nameref*{appendix:inner_loop}} \dotfill \pageref{appendix:inner_loop}
    \end{itemize}

    \item[] \hyperref[appendix:exp]{\ref*{appendix:exp}.~\nameref*{appendix:exp}} \dotfill \pageref{appendix:exp}
    \begin{itemize}
        \item[] \hyperref[appendix:image_denoise]{\ref*{appendix:image_denoise}.~\nameref*{appendix:image_denoise}} \dotfill \pageref{appendix:image_denoise}
        \item[] \hyperref[appendix:style_guide]{\ref*{appendix:style_guide}.~\nameref{appendix:style_guide}} \dotfill \pageref{appendix:style_guide}
    \end{itemize}

    \item[] \hyperref[appendix:computation_time]{\ref*{appendix:computation_time}.~\nameref{appendix:computation_time}} \dotfill \pageref{appendix:computation_time}

    \item[] \hyperref[appendix:code]{\ref*{appendix:code}.~\nameref*{appendix:code}} \dotfill \pageref{appendix:code}
    \item[] \hyperref[appendix:discussion]{\ref*{appendix:discussion}.~\nameref*{appendix:discussion}} \dotfill \pageref{appendix:discussion}

\end{itemize}

\newpage
\section{Notation}
\label{appendix:notation}
We provide a summary of the notations and symbols adopted in this work, as listed in the following.
\begin{table}[H]
    \centering
    \footnotesize{
    \begin{tabular*}{0.9\linewidth}{@{\extracolsep{\fill}} ll}
        \midrule
        \multicolumn{2}{l}{\textit{General Mathematics}} \\
        \midrule
        $\bm{I}_n$ & $n$-dimensional identity matrix \\
        $\mathbb{S}^{n-1}$ & $(n\!-\!1)$-dimensional hypersphere in $\mathbb{R}^n$ \\
        $\mathcal{N}(\bm{\mu}, \bm{\Sigma})$ & Gaussian distribution with mean $\bm{\mu}$ and covariance $\bm{\Sigma}$ \\
        $\operatorname{Unif}(A)$ & Uniform distribution on set $A$ \\
        $\mathcal{L}$ & Loss function \\
        $\nabla_{\bm{x}} \mathcal{L}$ & Euclidean gradient of $\mathcal{L}$ at $\bm{x}$ \\
        \midrule
        \multicolumn{2}{l}{\textit{Diffusion Process}} \\
        \midrule
        $\bm{x}_t$ & Noisy latent variable at diffusion timestep $t$ \\
        $\bar{\alpha}_t$ & DDIM noise schedule parameter \\
        $\bm{\mu}_t$ & Predicted mean of DDIM at timestep $t$ \\
        $\sigma_t$ & Noise standard deviation at timestep $t$ \\
        $\bm{\epsilon}_\theta(\bm{x}_t, t)$ & Output of diffusion model (e.g., UNet) \\
        $\hat{\bm{x}}_0(\bm{x}_t, t)$ & Clean image predicted by Tweedie’s formula \\
        \midrule
        \multicolumn{2}{l}{\textit{Diffusion Guidance}} \\
        \midrule
        $\bm{y}$ &  Conditional input or measurement \\
        $\mathcal{A}$ & Task-specific linear degradation operator \\
        $\eta_t$ & Guidance strength at timestep $t$ \\
        \midrule
        \multicolumn{2}{l}{\textit{Riemannian Geometry}} \\
        \midrule
        $\mathcal{S}_{t, r}$ & Spherical manifold at timestep $t$ \\
        $\operatorname{T}_{\bm{x}} \mathcal{M}$ & Tangent space of manifold $\mathcal{M}$ at point $\bm{x}$ \\
        $\Pi_{\operatorname{T}_{\bm{x}} \mathcal{M}}$ & Projection operator onto $\operatorname{T}_{\bm{x}} \mathcal{M}$ \\
        $\operatorname{R}_{\bm{x}}(\bm{v})$ & Retraction operator at point $\bm{x}$ along vector $\bm{v}$ \\
        $\operatorname{grad}_{\mathcal{M}} \bm{f}$ & Riemannian gradient of $\bm{f}$ on manifold $\mathcal{M}$ \\
        \bottomrule
    \end{tabular*}
    }
\end{table}

\section{Mathematical Details}
\label{appendix:math}

\setcounter{proposition}{0}
\setcounter{theorem}{0}

{
\renewcommand{\theproposition}{1}
\begin{proposition}[Polar Decomposition of Isotropic Gaussian]~\label{props:polar_decom}
At each diffusion timestep $t-1$, the latent $\bm{x}_{t-1} \sim \mathcal{N}(\bm{\mu}_t, \sigma_t^2 \bm{I}_n)$ follows an isotropic Gaussian distribution centered at $\bm{\mu}_t$. It admits a polar decomposition:
\begin{equation*}
\bm{x}_{t-1} = \bm{\mu}_t + \sigma_t r \bm{u},
\end{equation*}
where $r \sim \chi(n)$ is a chi-distributed random variable with $n$ degrees of freedom, and $\bm{u} \sim \operatorname{Unif}(\mathbb{S}^{n-1})$ is a unit vector uniformly distributed on the sphere. $r$ and $\bm{u}$ are statistically independent.
\end{proposition}

}

\begin{proof}
We can write $\bm{x}_{t-1} = \bm{\mu}_t + \sigma_t \bm{z}$, where $\bm{z} \sim \mathcal{N}(\bm{0}, \bm{I})$ is a standard multivariate Gaussian. Our goal is to show that $\bm{x}_{t-1}$ admits a decomposition as
\begin{equation*}
\bm{x}_{t-1} = \bm{\mu}_t + \sigma_t r \bm{u},
\end{equation*}
where $r \sim \chi(n)$ and $\bm{u} \sim \operatorname{Unif}(\mathbb{S}^{n-1})$ are independent.

Let the radius
\begin{equation*}
r = \|\bm{z}\|_2 = \sqrt{z_1^2 + \cdots + z_n^2},
\end{equation*}
and the unit vector
\begin{equation*}
\bm{u} = \frac{\bm{z}}{\|\bm{z}\|_2} \in \mathbb{S}^{n-1}.
\end{equation*}
Then $\bm{z}$ can be written as $\bm{z} = r \bm{u}$ in spherical coordinates. The joint density of $\bm{z}$ is given by
\begin{equation*}
p_{\bm{z}}(\bm{z}) = (2\pi)^{-n/2} \exp\left(-\frac{1}{2} \|\bm{z}\|_2^2\right).
\end{equation*}

Moreover, in spherical coordinates, the volume element transforms as
\begin{equation*}
\dd \bm{z} = r^{n-1} \dd r \, \dd \bm{u},
\end{equation*}
where $\dd \bm{u}$ denotes the uniform surface measure on the unit sphere $\mathbb{S}^{n-1}$. Therefore, the joint density in $(r, \bm{u})$ coordinates becomes
\begin{equation*}
p_{r, \bm{u}}(r, \bm{u}) = (2\pi)^{-n/2} \exp\left(-\frac{1}{2} r^2\right) r^{n-1}.
\end{equation*}
Marginalizing over $\bm{u}$ gives the marginal density of $r$:
\begin{align*}
p_r(r) &= \int_{\bm{u} \in \mathbb{S}^{n-1}} p_{r, \bm{u}}(r, \bm{u}) \, \dd \bm{u} = p_{r, \bm{u}}(r, \bm{u}) \cdot |\mathbb{S}^{n-1}| \\
&= \frac{1}{2^{n/2 - 1} \mathrm{\Gamma}(n/2)} r^{n-1} \exp\left(-\frac{1}{2} r^2\right),
\end{align*}
which is the probability density function of the $\chi(n)$ distribution. Moreover, the conditional density of $\bm{u}$ given $r$ is uniform over the sphere:
\begin{equation*}
p_{\bm{u}|r}(\bm{u} | r) = \frac{p_{r,\bm{u}}(r, \bm{u})}{p_r(r)} = \frac{1}{|\mathbb{S}^{n-1}|} = p_{\bm{u}}(\bm{u}).
\end{equation*}
Therefore, $\bm{x}_{t-1}$ admits the decomposition
\begin{equation*}
\bm{x}_{t-1} = \bm{\mu}_t + \sigma_t r \bm{u},
\end{equation*}
with $r \sim \chi(n)$, $\bm{u} \sim \operatorname{Unif}(\mathbb{S}^{n-1})$, and $r$ and $\bm{u}$ are independent.
\end{proof}

\begin{proposition}
\label{prop:chi_mean_var}
Let $r \sim \chi(n)$ be a chi-distributed random variable with $n$ degrees of freedom. Then, as $n \to \infty$, the expectation and variance of $r$ admit the following asymptotic expansions:
\begin{align*}
    \mathbb{E}[r] &= \sqrt{n - \tfrac{1}{2}} + \mathcal{O}(n^{-3/2}), \\
    \operatorname{Var}[r] &= \tfrac{1}{2} + \mathcal{O}(n^{-1}).
\end{align*}
\end{proposition}

\begin{proof}

The closed-form expression for the expectation of \( r \) is given by
\begin{equation*}
\mathbb{E}[r] = \sqrt{2} \cdot \frac{\mathrm{\Gamma}\left( \frac{n+1}{2} \right)}{\mathrm{\Gamma}\left( \frac{n}{2} \right)}.
\end{equation*}

To derive its asymptotic expansion, we apply the Legendre duplication formula:
\begin{equation*}
\frac{\mathrm{\Gamma}(z + \tfrac{1}{2})}{\mathrm{\Gamma}(z)} = 2^{1 - 2z} \sqrt{\pi} \cdot \frac{\mathrm{\Gamma}(2z)}{\mathrm{\Gamma}(z)^2}.
\end{equation*}
We next apply Stirling’s approximation to second order:
\begin{equation*}
\mathrm{\Gamma}(z) = \sqrt{\frac{2\pi}{z}} \left( \frac{z}{e} \right)^z \left( 1 + \frac{1}{12z} + \mathcal{O}(z^{-2}) \right).
\end{equation*}
Combining the above, we obtain
\begin{align*}
\frac{\mathrm{\Gamma}(z + \tfrac{1}{2})}{\mathrm{\Gamma}(z)}
&= \sqrt{z} \cdot \frac{1 + \frac{1}{24z} + \mathcal{O}(z^{-2})}{1 + \frac{1}{6z} + \mathcal{O}(z^{-2})} \\
&= \sqrt{z} \left( 1 - \frac{1}{8z} + \mathcal{O}(z^{-2}) \right).
\end{align*}
Substituting $z = n / 2$ yields
\begin{align*}
\mathbb{E}[r]
&= \sqrt{n} \left( 1 - \frac{1}{4n} + \mathcal{O}(n^{-2}) \right) \\
&= \sqrt{n - \tfrac{1}{2}} + \mathcal{O}(n^{-3/2}),
\end{align*}
where the final equality follows from the Taylor expansion of $\sqrt{n - \tfrac{1}{2}}$.

The variance of $r$ is computed via the identity:
\begin{align*}
\operatorname{Var}[r] = n - \mathbb{E}[r]^2.
\end{align*}
Using the expansion of $\mathbb{E}[r]$, we compute
\begin{align*}
\mathbb{E}[r]^2
&= \left[ \sqrt{n} \left( 1 - \frac{1}{4n} + \mathcal{O}(n^{-2}) \right) \right]^2 \\
&= n \left( 1 - \frac{1}{2n} + \mathcal{O}(n^{-2}) \right) \\
&= n - \tfrac{1}{2} + \mathcal{O}(n^{-1}).
\end{align*}
Therefore, the variance becomes
\begin{equation*}
\operatorname{Var}[r] = n - \left( n - \tfrac{1}{2} + \mathcal{O}(n^{-1}) \right) = \tfrac{1}{2} + \mathcal{O}(n^{-1}).
\end{equation*}

\end{proof}

\paragraph{Remark.}
$\sigma_t r$ will concentrate around $\sqrt{n} \sigma_t$ and the variance converges to $\sigma_t^2 / 2 \neq 0$ as $n \to \infty$. In contrast, DSG~\cite{yang2024dsg} fixed their spherical constraint with $\sqrt{n} \sigma_t$. Therefore, our formulation is more flexible.

\begin{definition}[Stationary Point on Riemannian Manifold]
Let $\bm{f}: \mathcal{M} \to \mathbb{R}$ be a differentiable function on a Riemannian manifold $\mathcal{M}$. A point $\bm{x}^* \in \mathcal{M}$ is a \emph{stationary point} if the Riemannian gradient vanishes, i.e.,
\begin{equation*}
    \operatorname{grad}_\mathcal{M} \bm{f}(\bm{x}^*) = \bm{0}.
\end{equation*}
\end{definition}

{
\renewcommand{\thetheorem}{1}
\begin{theorem}[Convergence to Stationary Point of DiffRGD] \label{theorem:convergence}
Let $\bm{f}_t : \mathcal{S}_{t, r} \to \mathbb{R}$ be the guidance objective at timestep $t$, defined as $\bm{f}_t(\bm{x}) := \mathcal{L}(\hat{\bm{x}}_0(\bm{x}, t-1), \bm{y})$, and let $\{ \bm{x}_{t-1}^{(k)} \}$ denote the sequence produced by Algorithm 1 at timestep $t$. Assume that $\bm{f}_t$ is $L_t$-smooth on the spherical manifold $\mathcal{S}_{t, r}$, and the guidance strength $0< \eta_\mathrm{min} \leq \eta_{t}^{(k)} \leq \frac{1}{L_t}$. Then the norm of the Riemannian gradient satisfies
\begin{equation*}
    \left\|\operatorname{grad}_{\mathcal{S}_{t, r}} \bm{f}_t(\bm{x}_{t-1}^{(k)}) \right\|_2 \to 0 \text{ as } k \to \infty.
\end{equation*}
with a convergence rate $\mathcal{O}(\frac{1}{\sqrt{k}})$.
\end{theorem}

}

\begin{proof}
Since $\bm{f}$ is $L_t$-smooth, the standard descent lemma for Riemannian Gradient Descent~\citep{absil2008rgd} implies:
\begin{equation*}
    \bm{f}_t(\bm{x}^{(k+1)}) \leq \bm{f}_t(\bm{x}^{(k)}) - \eta_t^{(k)} (1 - \tfrac{L_t \eta_t^{(k)}}{2}) \| \operatorname{grad}_{\mathcal{S}_{t, r}} \bm{f}_t(\bm{x}^{(k)}) \|_2^2.
\end{equation*}
Because $\eta_t^{(k)} \leq \frac{1}{L}_t$, we have $1 - \frac{L_t \eta_t^{(k)}}{2} \geq \frac{1}{2}$. Hence, for all $K \in \mathbb{N}$ we have
\begin{equation*}
    \bm{f}_t(\bm{x}^{(0)}) - \inf_{\bm{x} \in \mathcal{S}_{t, r}} \bm{f}_t(\bm{x}) \geq \frac{1}{2} \sum_{k=0}^{K-1} \eta_t^{(k)} \| \operatorname{grad}_{\mathcal{S}_{t, r}} \bm{f}_t(\bm{x}^{(k)}) \|_2^2.
\end{equation*}
Since $\eta_t^{(k)} \geq \eta_\mathrm{min} > 0$, dividing both sides by $K$, we obtain
\begin{equation*}
    \frac{1}{K} \sum_{k=0}^{K-1} \| \operatorname{grad}_{\mathcal{S}_{t, r}} \bm{f}_t(\bm{x}^{(k)}) \|_2^2 \leq \frac{2\left( \bm{f}_t(\bm{x}^{(0)}) - \displaystyle \inf_{\bm{x} \in \mathcal{S}_{t, r}} \bm{f}_t(\bm{x}) \right)}{\eta_\mathrm{min} K}.
\end{equation*}
Therefore, there exists integer $k^* \in [0, K]$ such that
\begin{equation*}
    \| \operatorname{grad}_{\mathcal{S}_{t, r}} \bm{f}_t(\bm{x}^{(k^*)}) \|_2 \leq \sqrt{ \frac{2\left( \bm{f}_t(\bm{x}^{(0)}) - \inf_{\bm{x} \in \mathcal{S}_{t, r}} \bm{f}_t(\bm{x}) \right)}{\eta_\mathrm{min} K} }.
\end{equation*}
This implies that $\| \operatorname{grad}_{\mathcal{S}_{t, r}} \bm{f}_t(\bm{x}^{(k)}) \|_2 \to 0$ as $k \to \infty$ and the convergence rate is $O(\frac{1}{\sqrt{k}})$.
\end{proof}

\section{Implementation Details}
\label{appendix:implement}

\subsection{Algorithms of Compared Methods}
\label{appendix:algos}

We provide the algorithms of previous methods used in our comparisons: DPS~\citep{chung2023dps} in Algorithm~\ref{algo:dps}, MPGD~\citep{he2024mpgd} in Algorithm~\ref{algo:mpgd}, DSG~\citep{yang2024dsg} in Algorithm~\ref{algo:dsg}, and ADMMDiff~\citep{zhang2025admmdiff} in Algorithm~\ref{algo:admmdiff}. We also correct several typos in the ADMMDiff algorithm from the original paper.

\setcounter{algorithm}{1}
\begin{algorithm}[H]
    \caption{DPS~\citep{chung2023dps}}~\label{algo:dps}
    \begin{algorithmic}[1] 
        \Require{
        Reference image $\bm{y} \in \mathbb{R}^n$,
        Diffusion model $\bm{\epsilon}_\theta$,
        loss function $\mathcal{L}$,
        DDIM sampling schedule parameters $\{\bar{\alpha}_t\}_{t=1}^T$,
        noise level $\sigma_t > 0$,
        guidance strengths $\{\eta_t\}_{t=1}^T$.
        }
        \State{$\bm{x}_T \sim \mathcal{N}(\bm{0}, \bm{I}_n)$.}
        \For{$t=T$ to $1$}
            \State{$\bm{\epsilon}_t \sim \mathcal{N}(\bm{0}, \bm{I}_n)$.}
            \State{$\bm{\mu}_{t} \leftarrow \sqrt{\bar{\alpha}_{t - 1}}
            \hat{\bm{x}}_0(\bm{x}_t, t) + \sqrt{1-\bar{\alpha}_{t-1}-\sigma_t^2}\bm{\epsilon}_{\theta}(\bm{x}_t, t)$.}
            \State{$\hat{\bm{x}}_{t-1} \leftarrow \bm{\mu}_{t} + \sigma_t \bm{\epsilon}_t$.}
            \State{$\bm{x}_{t-1} \leftarrow \hat{\bm{x}}_{t-1} - \eta_t \nabla_{\bm{x}_{t}} \mathcal{L}(\hat{\bm{x}}_{0}(\bm{x}_{t}, t), \bm{y})$.}
        \EndFor
        \Ensure{Guided sample $\bm{x}_0$.}
    \end{algorithmic}
\end{algorithm}

\begin{algorithm}[H]
    \caption{MPGD~\citep{he2024mpgd}}~\label{algo:mpgd}
    \begin{algorithmic}[1] 
        \Require{
        Reference image $\bm{y} \in \mathbb{R}^n$,
        Diffusion model $\bm{\epsilon}_\theta$,
        projector $\mathcal{P}_\mathcal{M}$,
        loss function $\mathcal{L}$,
        DDIM sampling schedule parameters $\{\bar{\alpha}_t\}_{t=1}^T$,
        noise level $\sigma_t > 0$,
        guidance strengths $\{\eta_t\}_{t=1}^T$.
        }
        \State{$\bm{x}_T \sim \mathcal{N}(\bm{0}, \bm{I}_n)$.}
        \For{$t=T$ to $1$}
            \State{$\bm{\epsilon}_t \sim \mathcal{N}(\bm{0}, \bm{I}_n)$.}
            \State{$\bm{x}_{0 \mid t} \leftarrow \hat{\bm{x}}_0(\bm{x}_t, t)$.}
            \State{$\hat{\bm{x}}_{0 \mid t} \leftarrow \bm{x}_{0 \mid t} - \eta_t \mathcal{P}_\mathcal{M}(\nabla_{\bm{x}_{0 \mid t}} \mathcal{L}(\bm{x}_{0 \mid t}, \bm{y}))$.
            \Comment{PGD}
            }
            \State{$\hat{\bm{\epsilon}}_t \leftarrow \frac{\bm{x}_t - \sqrt{\bar{\alpha}_t} \hat{\bm{x}}_{0|t}}{\sqrt{1-\bar{\alpha}_t}}$.}
            \State{$\bm{\mu}_{t} \leftarrow \sqrt{\bar{\alpha}_{t - 1}}
            \hat{\bm{x}}_{0 \mid t} + \sqrt{1-\bar{\alpha}_{t-1}-\sigma_t^2}\hat{\bm{\epsilon}}_t$.}
            \State{$\bm{x}_{t-1} \leftarrow \bm{\mu}_{t} + \sigma_t \bm{\epsilon}_t$.}
        \EndFor
        \Ensure{Guided sample $\bm{x}_0$.}
    \end{algorithmic}
\end{algorithm}

\begin{algorithm}[H]
    \caption{DSG~\citep{yang2024dsg}}~\label{algo:dsg}
    \begin{algorithmic}[1] 
        \Require{
        Reference image $\bm{y} \in \mathbb{R}^n$,
        Diffusion model $\bm{\epsilon}_\theta$,
        loss function $\mathcal{L}$,
        DDIM sampling schedule parameters $\{\bar{\alpha}_t\}_{t=1}^T$,
        noise level $\sigma_t > 0$,
        guidance strengths $\{\eta_t\}_{t=1}^T$.
        }
        \State{$\bm{x}_T \sim \mathcal{N}(\bm{0}, \bm{I}_n)$.}
        \For{$t=T$ to $1$}
            \State{$\bm{\epsilon}_t \sim \mathcal{N}(\bm{0}, \bm{I}_n)$.}
            \State{$\bm{u} \leftarrow \sigma_t \bm{\epsilon}_t$.}
            \State{$\bm{g} \leftarrow \nabla_{\bm{x}_{t}} \mathcal{L}(\hat{\bm{x}}_{0}(\bm{x}_t, t), \bm{y}) \, / \, \|\nabla_{\bm{x}_{t}} \mathcal{L}(\hat{\bm{x}}_{0}(\bm{x}_t, t), \bm{y})\|_2$.}
            \State{$\bm{\mu}_t \leftarrow \sqrt{\bar{\alpha}_{t-1}} \hat{\bm{x}}_{0}(\bm{x}_t, t)+ \sqrt{1 - \bar{\alpha}_{t-1} - \sigma_t^2} \, \bm{\epsilon}_\theta(\bm{x}_t, t)$.}
            \State{
            $\bm{d} \leftarrow (1-\eta_t)\bm{u} + \eta_t \bm{g}$.
            \hfill $\triangleright$ Interpolation
            }
            \State{$\bm{x}_{t-1} \leftarrow \bm{\mu}_{t} + \sqrt{n} \sigma_t \frac{\bm{d}}{\|\bm{d}\|_2}$.
            \hfill $\triangleright$ Normalization
            }
        \EndFor
        \Ensure{Guided sample $\bm{x}_0$.}
    \end{algorithmic}
\end{algorithm}

\begin{algorithm}[H]
    \caption{ADMMDiff~\citep{zhang2025admmdiff}}~\label{algo:admmdiff}
    \begin{algorithmic}[1] 
        \Require{
        Reference image $\bm{y} \in \mathbb{R}^n$,
        Diffusion model $\bm{\epsilon}_\theta$,
        loss function $\mathcal{L}$,
        DDIM sampling schedule parameters $\{\bar{\alpha}_t\}_{t=1}^T$,
        noise level $\sigma_t > 0$,
        guidance strengths $\{\eta_t^{(k)}\}_{t=1}^T$,
        guidance inner iteration $K \in \mathbb{N}$,
        penalty parameter $\rho > 0$.
        }
        \State{$\bm{x}_T \sim \mathcal{N}(\bm{0}, \bm{I}_n)$.}
        \State{$\bm{z}_T \leftarrow \bm{x}_T$.}
        \State{$\bm{v}_T \leftarrow \bm{0}$.}
        \For{$t=T$ to $1$}
            \State{$\bm{\epsilon}_t \sim \mathcal{N}(\bm{0}, \bm{I}_n)$.}
            \State{$\hat{\bm{x}}_{t} \leftarrow \bm{z}_t -\frac{1}{\rho}\bm{v}_t$.
            \Comment{$\bm{x}_t$-subproblem}
            }
            \State{$\bm{\mu}_{t} \leftarrow \sqrt{\bar{\alpha}_{t - 1}}
            \hat{\bm{x}}_0(\hat{\bm{x}}_{t}, t) + \sqrt{1-\bar{\alpha}_{t-1}-\sigma_t^2}\bm{\epsilon}_{\theta}(\hat{\bm{x}}_{t}, t)$.}
            \State{$\bm{x}_{t-1} \leftarrow \bm{\mu}_{t} + \sigma_t \bm{\epsilon}_t$.}
            \State{$\bm{z}_t^{(0)} \leftarrow \bm{z}_t$.
            \Comment{$\bm{z}_t$-subproblem}
            }
            \For{$k=0$ to $K-1$}
                \State{$\bm{g} \leftarrow \nabla_{\bm{z}_t^{(k)}} \mathcal{L}(\hat{\bm{x}}_{0}(\bm{z}_t^{(k)}), \bm{y})$.}
                \State{$\bm{z}_t^{(k+1)} \leftarrow \bm{z}_t^{(k)} - \eta_t^{(k)} \left[ \bm{g} + \rho (\bm{z}_t^{(k)} - \bm{x}_{t-1} - \bm{v}_t) \right]$.}
            \EndFor
            \State{$\bm{z}_{t-1} \leftarrow \bm{z}_t^{(K)}$.}
            \State{$\bm{v}_{t-1} \leftarrow \bm{v}_{t} + \rho (\bm{x}_{t-1}-\bm{z}_{t-1})$.}
        \EndFor
        \Ensure{Guided sample $\bm{x}_0$.}
    \end{algorithmic}
\end{algorithm}

\subsection{Hyperparameter Settings}
\label{appendix:params}

We summarize the hyperparameter configurations used in our experiments. Table~\ref{tab:hyperparams_ffhq} lists the settings for image restoration tasks on FFHQ $256 \times 256$~\citep{karras2019ffhq}, and Table~\ref{tab:hyperparams_celeb_hq} lists those for the conditional generation tasks on CelebA-HQ $256 \times 256$~\citep{karras2018celeba}.

\subsection{Evaluation Metrics}
\label{appendix:metrics}

\paragraph{Image Restoration.}
We use Peak Signal-to-Noise Ratio (PSNR)~\citep{jain1989dip}, Structural Similarity Index Measure (SSIM)~\citep{wang2004ssim}, Learned Perceptual Image Patch Similarity (LPIPS)~\citep{zhang2018lpips}, and Fréchet Inception Distance (FID)~\citep{heusel2017fid} as evaluation metrics.  
PSNR assesses pixel-wise fidelity, while SSIM captures structural similarity based on local luminance and contrast.  
LPIPS evaluates perceptual similarity using deep features from pre-trained networks, better reflecting human judgment.  
FID measures distributional alignment between generated and real images in the Inception feature space.

\paragraph{Conditional Generation.}
We adopt FID and Kernel Inception Distance (KID)~\citep{binkowski2018kid} for all conditional generation tasks.  
For segmentation map guidance, we report the mean Intersection-over-Union (mIoU)~\citep{long2015semanticsegment} between the predicted segmentation maps of the generated and reference images.  
For sketch and FaceID guidance, we use the $\ell_2$ distance between the outputs of pre-trained parsing models applied to the generated and reference images.

\section{Detailed Analysis of DiffRGD}
\label{appendix:analysis}

\subsection{Distribution Preservation}
\label{appendix:distribution}

To evaluate distribution preservation, we design an experiment grounded in the forward diffusion process.  Recall that under forward diffusion, the latent variable follows $\bm{x}_t^{\mathrm{fwd}} 
\sim \mathcal{N}\bigl(\sqrt{\bar{\alpha}_t}\bm{x}_0,\; (1-\bar{\alpha}_t)\bm{I}_n\bigr)$. Standardizing the noise component yields
\begin{equation*}
    \frac{\bm{x}_t^{\mathrm{fwd}} - \sqrt{\bar{\alpha}_t}\bm{x}_0}
         {\sqrt{1-\bar{\alpha}_t}}
    \sim 
    \mathcal{N}(\bm{0}, \bm{I}_n),
    \text{ which implies }
    \left\|
    \frac{\bm{x}_t^{\mathrm{fwd}} - \sqrt{\bar{\alpha}_t}\bm{x}_0}
         {\sqrt{1-\bar{\alpha}_t}}
    \right\|_2^2
    \sim 
    \chi^2(n).
\end{equation*}

In the ideal case, a guidance method that preserves the sampling distribution of $\bm{x}_t^\text{guided}$ should not alter the statistics of the injected noise. For simplicity, we can check whether $e:=\mathbb{E}\left[\tfrac{1}{n}\|\tfrac{\bm{x}_t^\text{guided}-\sqrt{\bar{\alpha}_t}\bm{x}_0}{\sqrt{1-\bar{\alpha}_t}}\|_2^2\right] = 1$ during guided sampling. We evaluate five timesteps and report their empirical mean of $e$ over 100 samples. The table below shows that DiffRGD and DPS are comparable at earlier timesteps, while at mid-to-late timesteps DiffRGD is noticeably closer to 1 than DPS. The results indicate that DiffRGD can preserve the intended per-step geometry better, thereby reducing accumulated drift along the sample trajectory compared to DPS.

\begin{table*}[h]
    \centering
    \setlength{\tabcolsep}{6pt}
    \caption{Distribution preservation comparison between DPS and DiffRGD.}
    \begin{tabular}{lcccccc}
        \toprule
        DDIM Step & $t=100$ & $t=80$ & $t=60$ & $t=40$ & $t=20$ \\
        \midrule
        $e_\text{DPS}$ & $\first{0.9998}$ & $\first{0.9999}$ & $1.0119$ & $1.1017$ & $1.8007$ \\
        \rowcolor{blue!10}
        $e_\text{DiffRGD}$ & $0.9996$ & $1.0002$ & $\first{1.0063}$ & $\first{1.0480}$ & $\first{1.3743}$ \\
        \bottomrule
    \end{tabular}
\end{table*}

\subsection{Empirical Convergence of the Riemannian Gradient Norm}
\label{appendix:inner_loop}

Figure~\ref{fig:gradient_norm} illustrates the relationship over the number of inner iterations and the norm of the Riemannian gradient at different timesteps. While Theorem~\ref{theorem:convergence} guarantees a convergence rate of $\mathcal{O}(1 / \sqrt{k})$, the figure shows that the norm of the Riemannian gradient already becomes sufficiently small when $K = 3$ in practice. This supports our experiment settings in 
Table~\ref{tab:hyperparams_ffhq} and Table~\ref{tab:hyperparams_celeb_hq}.

\begin{figure}[H]
    \centering
    \includegraphics[width=0.7\linewidth]{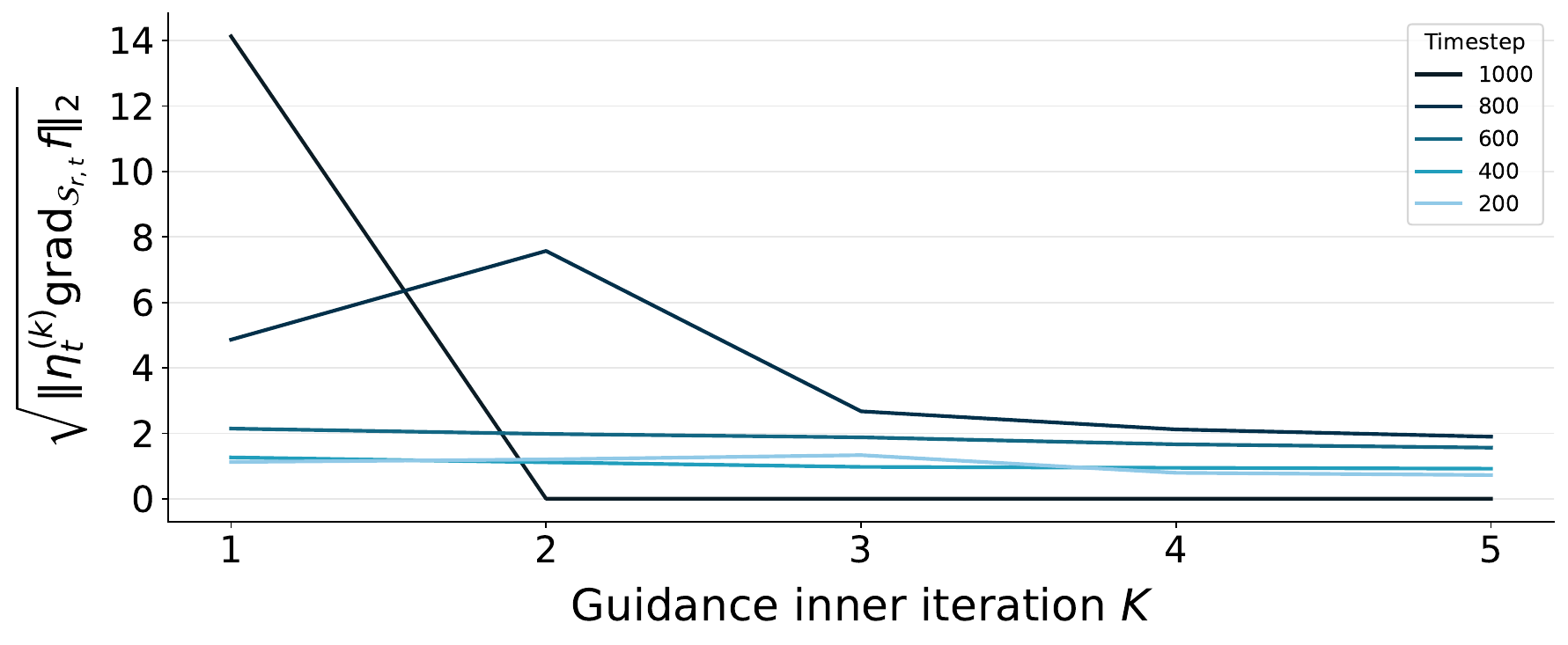}
    \caption{Norm of Riemannian gradient versus guidance inner iteration across different timesteps.
    For visualization purposes, we plot the square root of the gradient norm.
    }
    \label{fig:gradient_norm}
    \vspace*{-10pt}
\end{figure}

\section{Additional Results}
\label{appendix:exp}

We present additional qualitative results for image restoration in Figure~\ref{fig:appendix_exp_image_restoration} and for conditional generation in Figure~\ref{fig:appendix_exp_control}. We also provide qualitative results under different timestep settings in Figure~\ref{fig:exp_different_timestep_combined} for comparison. Our method can perform well even with 100 timesteps. These results are consistent with the quantitative results in Table~\ref{tab:ffhq_ddim100_2}, and Table~\ref{tab:ffhq_ddim100_1} of the main paper. Additional experiments are presented in the following subsections.

\begin{table*}[h]
    \centering
    \setlength{\tabcolsep}{6pt}
    \caption{Quantitative comparison of DDIM 100 sampling steps on 100 samples from the FFHQ $256 \times 256$ validation set for random inpainting and super-resolution tasks.
    }
    \resizebox{0.9\linewidth}{!}{
    \centering
    \begin{tabular}{lcccccc}
        \toprule
        \multirow{2.5}{*}{\textbf{Method}} & \multicolumn{3}{c}{\textbf{Random Inpainting}} & \multicolumn{3}{c}{\textbf{Super-Resolution $4 \times$}} \\
        \cmidrule(lr){2-4} \cmidrule(lr){5-7}
        ~ & \textbf{PSNR $\uparrow$} & \textbf{SSIM $\uparrow$} & \textbf{LPIPS $\downarrow$} & \textbf{PSNR $\uparrow$} & \textbf{SSIM $\uparrow$} & \textbf{LPIPS $\downarrow$} \\
        \midrule
        DPS~\citep{chung2023dps} & $26.09$ & $0.765$ & $0.266$ & $22.46$ & $0.643$ & $0.345$ \\
        MPGD~\citep{he2024mpgd} & $28.13$ & $0.764$ & $0.232$ & $24.23$ & $0.618$ & $0.345$ \\    
        DSG~\citep{yang2024dsg} & $\second{31.25}$ & $\second{0.854}$ & $\second{0.155}$ & $\second{27.36}$ & $\second{0.753}$ & $\second{0.285}$ \\
        ADMMDiff~\citep{zhang2025admmdiff} & $29.09$ & $0.809$ & $0.235$ & $25.12$ & $0.709$ & $0.314$ \\
        \rowcolor{blue!10}
        DiffRGD (Ours) & $\first{34.29}$ & $\first{0.927}$ & $\first{0.127}$ & $\first{28.35}$ & $\first{0.806}$ & $\first{0.231}$ \\
        \bottomrule
    \end{tabular}
    }
    \label{tab:ffhq_ddim100_2}
\end{table*}

\subsection{Image Denoising}
\label{appendix:image_denoise}

Image denoising is a fundamental image restoration task that aims to recover clean images from noisy observations. Traditional approaches often rely on explicit image priors, such as total variation regularization~\citep{rudin1992tv, huang2024denoising}, while diffusion models possess inherent denoising capabilities through their learned generative priors. We compare the denoising performance of DPS~\citep{chung2023dps}, DSG~\citep{yang2024dsg}, and our proposed DiffRGD on ImageNet $256 \times 256$ validation dataset. The evaluation is conducted under three different noise levels. As shown in Table~\ref{tab:denoising}, DiffRGD consistently achieves the best performance across all noise settings. Qualitative results are illustrated in Figure~\ref{fig:denoising}. In the top example with $\sigma = 0.3$, both DPS and DSG exhibit noticeable residual noise, whereas DiffRGD restores the image with minimal artifacts. In the bottom example, DiffRGD reconstructs the mouse with significantly better visual quality than DPS and DSG.

\begin{table}[h]
    \centering
    \setlength{\tabcolsep}{6pt}
    \centering
    \caption{Quantitative comparison for denoising tasks on ImageNet 256 $\times$ 256.}
    \resizebox{0.9\linewidth}{!}{
    \begin{tabular}{lcccccc}
        \toprule
        \multirow{2.5}{*}{\textbf{Method}} & \multicolumn{2}{c}{$\sigma = 0.1$} & \multicolumn{2}{c}{$\sigma = 0.2$} & \multicolumn{2}{c}{$\sigma = 0.3$} \\
        \cmidrule(lr){2-3} \cmidrule(lr){4-5} \cmidrule(lr){6-7}
       ~ & \textbf{PSNR $\uparrow$} & \textbf{SSIM $\uparrow$} & \textbf{PSNR $\uparrow$} & \textbf{SSIM $\uparrow$} & \textbf{PSNR $\uparrow$} & \textbf{SSIM $\uparrow$} \\
        \midrule
        DPS~\citep{chung2023dps} & $25.19$ & $0.699$ & $23.14$ & $0.607$ & $20.98$ & $0.476$ \\
        DSG~\citep{yang2024dsg} & $27.80$ & $0.774$ & $25.70$ & $0.694$ & $24.05$ & $0.607$ \\
        \rowcolor{blue!10}
        DiffRGD (Ours) & $\first{33.13}$ & $\first{0.892}$ & $\first{30.66}$ & $\first{0.845}$ & $\first{29.13}$ & $\first{0.808}$ \\
        \bottomrule
    \end{tabular}
    }
    \label{tab:denoising}
\end{table}

\subsection{Style-Guided Generation}
\label{appendix:style_guide}

We evaluate our method on the style-guided generation task using the Stable Diffusion model~\citep{rombach2022ldm}. Given a style image $\bm{y}$, we extract task-specific features using a pre-trained CLIP encoder $\psi_\theta$~\citep{radford2021clip}. To measure stylistic alignment, we construct the Gram matrix $G(\cdot)$ of the extracted features. The style guidance loss is then defined as:
\begin{equation*}
\mathcal{L}_\text{style}(\hat{\bm{x}}_0, \bm{y}) = \|G(\psi_\theta(\hat{\bm{x}}_0)) - G(\psi_\theta(\bm{y}))\|_2.
\end{equation*}

\paragraph{Implementation Details.}
We select five distinct style images from the WikiArt~\citep{saleh2015wikiart} dataset and generate 20 animal-related prompts using GPT-5.1~\citep{openai2025gpt5}, resulting in 100 unique style-content pairs. All generations are performed using the DDIM sampler with 150 sampling steps. We use Stable Diffusion v1.5 from HuggingFace as the shared backbone for all methods. This choice differs from the original MPGD and FreeDoM settings, where the methods rely on LDMs trained on specific datasets\footnote{\url{https://github.com/KellyYutongHe/mpgd_pytorch/blob/main/nonlinear/SD_style/scripts/download_models.sh}}. However, we argue that practical inference-time guidance methods should be able to adapt to general foundation models rather than specialized dataset-dependent models.

\paragraph{Results.}
Table~\ref{tab:style_transfer} reveals that while all methods yield comparable CLIP scores, DiffRGD consistently achieves lower Style scores, indicating superior alignment with the target style. Figure~\ref{fig:style_transfer} further illustrates that DiffRGD produces generations that are both visually richer and stylistically more faithful than those of baseline methods.

\begin{table}[h]
    \caption{Quantitative comparison for style-guided generation with Stable Diffusion.}
    \centering
    \setlength{\tabcolsep}{6pt}
    \begin{tabular}{lcc}
        \toprule
       \textbf{Method} & \textbf{CLIP-Score $\uparrow$} & \textbf{Style-Score $\downarrow$} \\
        \midrule
        FreeDoM~\citep{yu2023freedom} & $0.183$ & $6.348$ \\ 
        MPGD~\citep{he2024mpgd} & $\first{0.186}$ & $6.018$  \\
        DSG~\citep{yang2024dsg} & $0.184$ & $5.151$ \\ 
        \rowcolor{blue!10}
        DiffRGD (Ours) & $0.185$ & $\first{4.848}$ \\ 
        \bottomrule
    \end{tabular}
    \label{tab:style_transfer}
    \vspace{-15pt}
\end{table}

\section{Computational Cost and Runtime}
\label{appendix:computation_time}

Let $F$ denote the runtime of a forward pass, $B$ the runtime of a backward pass, $P$ the runtime of projection, $T$ the number of DDIM sampling steps, $K$ the number of inner iterations, and $I$ the guidance interval. Table~\ref{tab:time_complexity} provides the forward/backward pass counts and time complexity of different methods. For MPGD, the complexity additionally accounts for the projection cost.

\begin{table}[h]
    \centering
    \caption{Forward/backward pass counts and time complexity of different methods.}
    \setlength{\tabcolsep}{6pt}
    \label{tab:time_complexity}
    \begin{tabular}{lccl}
        \toprule
        \textbf{Method} & \textbf{\#Forward} & \textbf{\#Backward} & \textbf{Time Complexity}\\
        \midrule
        DPS~\citep{chung2023dps} & $T$ & $T$ & $\mathcal{O}((F + B) T)$ \\
        MPGD~\citep{he2024mpgd} & $T$ & $T$ & $\mathcal{O}((F + B + P) T)$ \\
        DSG~\citep{yang2024dsg} & $T$ & $T/I$ & $\mathcal{O}((F + B / I) T)$ \\
        ADMMDiff~\citep{zhang2025admmdiff} & $(K+1)T$ & $KT$ & $\mathcal{O}((F (K+1) + BK ) T)$ \\
        \rowcolor{blue!10}
        DiffRGD (Ours) & $T + KT/I$ & $KT/I$ & $\mathcal{O}( (F + (F+B)K / I) T)$ \\
        \bottomrule
    \end{tabular}
\end{table}

\begin{table}[h]
    \setlength{\tabcolsep}{6pt}
    \centering
    \caption{Inference time per image in seconds for image restoration tasks.}
    \begin{tabular}{lccccc}
        \toprule
        Task & DPS & MPGD & DSG & ADMMDiff & DiffRGD \\
        \midrule
        Inpainting & 22.09 & 21.91 & 22.01 & 39.61 & 32.76 \\
        Super-Resolution $4 \times$ & 22.18 & 22.06 & 22.69 & 39.42 & 27.77 \\
        Gaussian Deblurring & 22.35 & 22.22 & 22.00 & 39.64 & 32.77 \\
        Motion Deblurring & 23.11 & 22.92 & 22.20 & 40.42 & 36.09 \\
        \midrule
        Avg. Time (Seconds) & 22.43 & 22.28 & 22.23 & 39.77 & 32.35 \\
        \bottomrule
    \end{tabular}
    \label{tab:image_restora_time}
\end{table}

\begin{table}[h]
    \setlength{\tabcolsep}{6pt}
    \centering
    \caption{Inference time per image in seconds for conditional generation tasks.}
    \begin{tabular}{lcccc}
        \toprule
        Task & FreeDoM & DSG & ADMMDiff & DiffRGD \\
        \midrule
        Segmentation Map & 4.35 & 4.36 & 45.40 & 16.18 \\
        Sketch & 4.17 & 4.18 & 43.59 & 15.68 \\
        FaceID & 4.90 & 4.90 & 50.34 & 17.72 \\
        \midrule
        Avg. Time (Seconds) & 4.47 & 4.48 & 46.44 & 16.53 \\
        \bottomrule
    \end{tabular}
    \label{tab:condition_gen_time}
\end{table}

All inference times are measured in seconds per image on a single NVIDIA GeForce RTX 4090 GPU. For image restoration tasks, we report the runtime corresponding to Table~\ref{tab:ffhq_ddim1000} of the main paper in Table~\ref{tab:image_restora_time}, where DiffRGD achieves a $\mathbf{1.23\times}$ speedup over ADMMDiff. For conditional generation tasks, we report the runtime corresponding to Table~\ref{tab:condition_gen} of the main paper in Table~\ref{tab:condition_gen_time}, where DiffRGD achieves a $\mathbf{2.8\times}$ speedup over ADMMDiff.

Although DiffRGD requires more computation time than DPS, MPGD, and DSG, it consistently outperforms these baselines in terms of task performance. Moreover, the qualitative results shown in Figure~\ref{fig:exp_condition} of the main paper demonstrate that DiffRGD produces superior conditional generation quality, highlighting its practical value for real-world applications.

\section{Code and Software Availability}
\label{appendix:code}

We integrate several representative inference-time guidance methods, including DPS\footnote{\url{https://github.com/DPS2022/diffusion-posterior-sampling}}, MPGD\footnote{\url{https://github.com/KellyYutongHe/mpgd_pytorch}}, DSG\footnote{\url{https://github.com/LingxiaoYang2023/DSG2024}}, and ADMMDiff \footnote{\url{https://github.com/youyuan-zhang/ADMMDiff}}, into the Diffusers framework~\citep{platen2022diffusers}, enabling flexible application of these guidance strategies across different tasks. During implementation, we identify and fix several issues in the original GitHub implementations to ensure faithful and stable comparisons. In particular, we find a discrepancy between the ADMMDiff codebase\footnote{\url{https://github.com/youyuan-zhang/ADMMDiff/blob/master/Linear/guided_diffusion/gaussian_diffusion.py}} and the algorithm described in the paper: the gradient inside the inner loop is not recomputed with respect to the updated sample. We correct this implementation and tune the corresponding hyperparameters for a fair comparison. We believe these improvements provide a useful resource for the open-source community.

\section{Limitation and Discussion}
\label{appendix:discussion}

\paragraph{Dependence on Strong Pre-trained Model.}
Inference-time guidance methods rely on a strong pre-trained diffusion model to provide meaningful guidance. In particular, DSG, ADMMDiff, and our DiffRGD impose constraints directly on the model-predicted mean, making their effectiveness dependent on the accuracy of the underlying diffusion model. When the model suffers from bias or produces inaccurate mean predictions, such errors can propagate to the final guided samples. MPGD is similarly affected, with additional potential bias introduced by the imperfect auxiliary autoencoder.

\paragraph{Computational Overhead from RGD Inner Loop.}
Our method incurs additional computational overhead due to its inner-loop Riemannian gradient descent. As a result, the inference speed is largely determined by the number of inner iterations required for effective guidance.

\paragraph{Limited Extension to Deterministic Sampling and Flow Matching.}
DiffRGD relies on the Gaussian transition structure of diffusion models. Therefore, it cannot be directly applied to flow matching or deterministic samplers, which do not generally preserve the same transition geometry. Extending DiffRGD to these settings would require redesigning the underlying geometric formulation for the corresponding generative process.

\paragraph{Future Work.}
In principle, the sampling radius $r$ could be incorporated into an alternating optimization scheme, where $r$ is jointly optimized while being constrained to remain close to the sampled radius. Such a strategy may potentially yield more optimal samples. However, this approach would introduce additional computational overhead. Exploring more principled optimization schemes for $r$ while maintaining computational efficiency remains an interesting direction for future work. In this paper, we instead adopt a simple and efficient heuristic by sampling $r$ to maintain practical efficiency.

\clearpage

\begin{figure}[t]
    \centering
    \includegraphics[width=\linewidth]{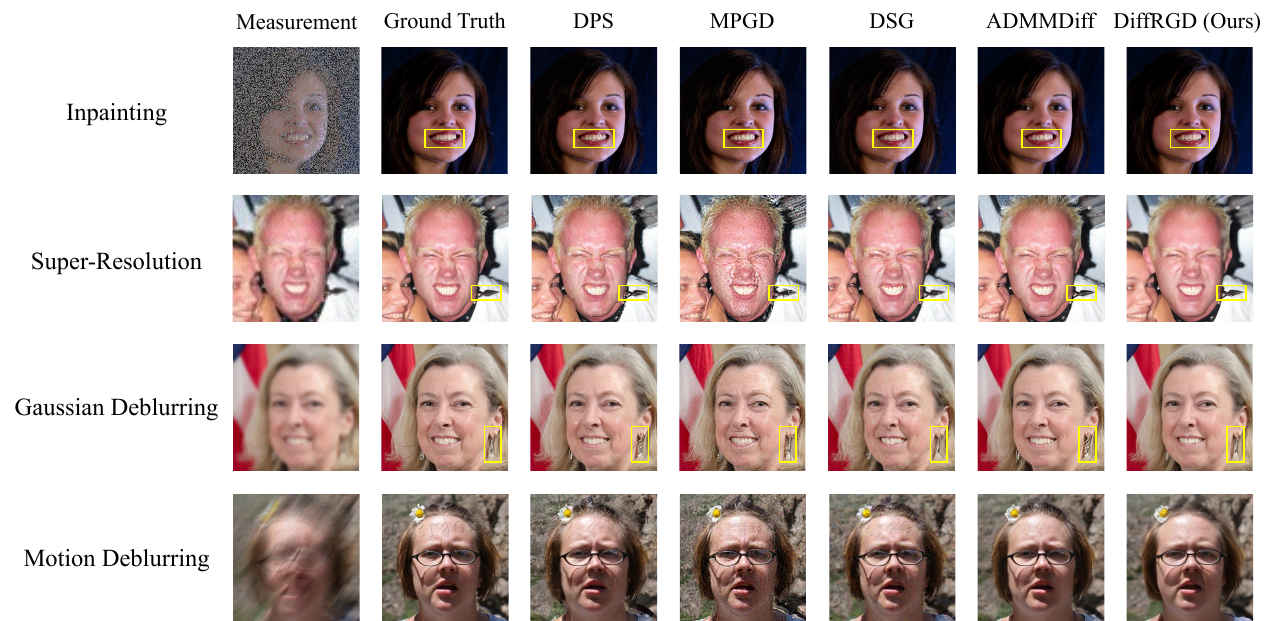}
    \caption{Qualitative comparisons for image restoration tasks on FFHQ 256 $\times$ 256 validation set. Please refer to the highlighted regions for detailed comparison.}
    \label{fig:appendix_exp_image_restoration}
\end{figure}

\begin{figure}[t]
    \centering
    \includegraphics[width=\linewidth]{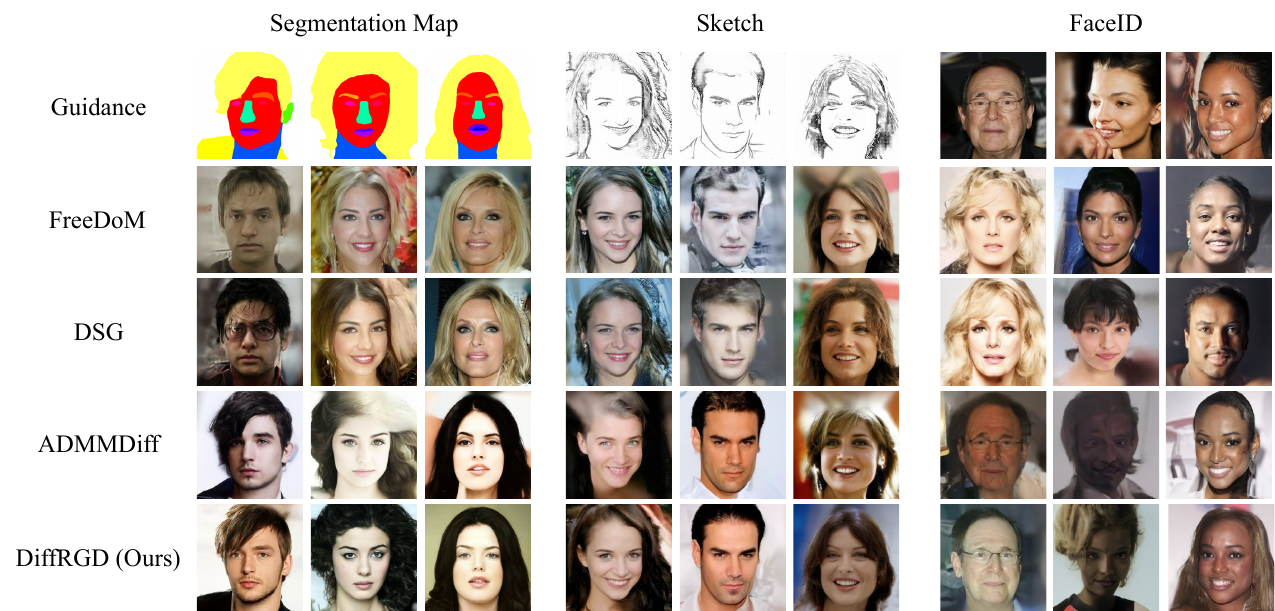}
    \caption{Qualitative comparison for conditional generation tasks on CelebA-HQ 256 × 256 validation set. (a) segmentation maps to human faces (b) sketches to human faces (c) FaceID to human faces.}
    \label{fig:appendix_exp_control}
\end{figure}

\begin{figure*}[h]
    \centering
    \begin{minipage}[t]{\linewidth}
        \centering
        \includegraphics[width=\linewidth]{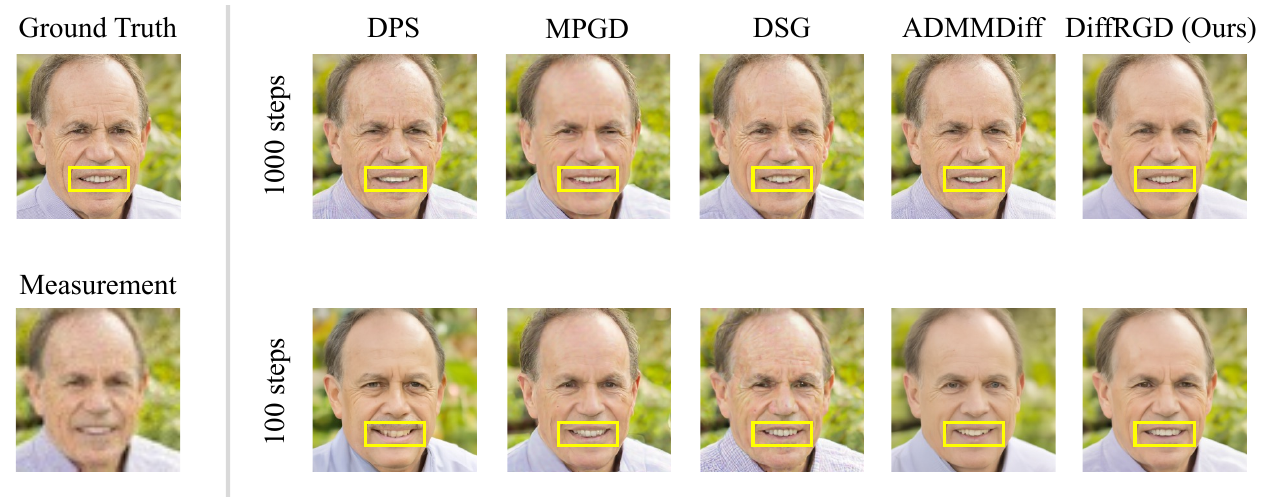}
        {(a) Super-resolution.}
        \label{fig:exp_different_timestep1}
    \end{minipage}

    \vspace{1em}

    \begin{minipage}[t]{\linewidth}
        \centering
        \includegraphics[width=\linewidth]{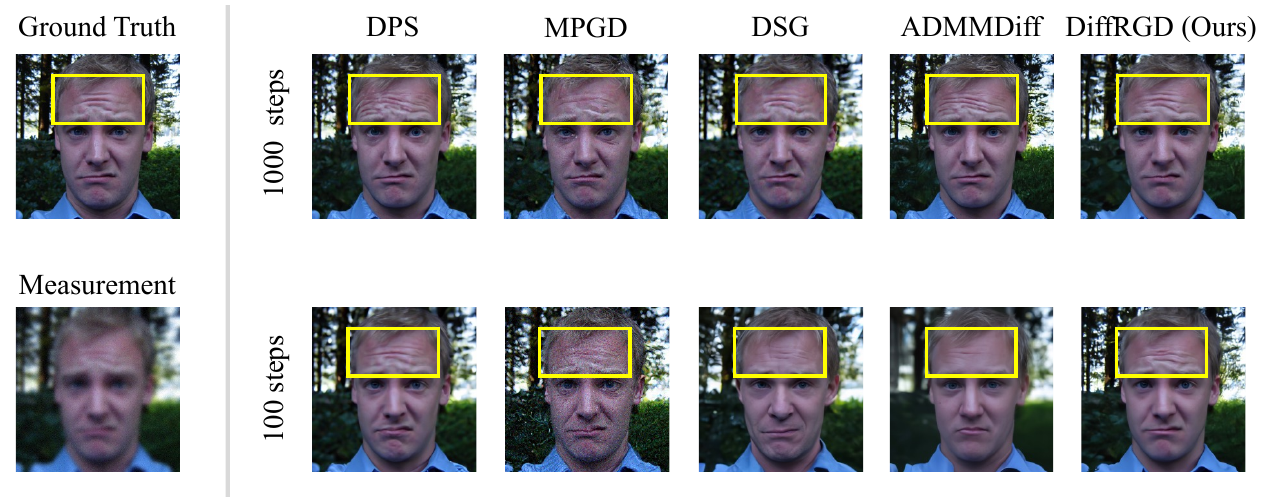}
        {(b) Gaussian deblurring.}
        \label{fig:exp_different_timestep2}
    \end{minipage}

    \caption{Qualitative comparisons on different DDIM timesteps: Compared to 1000 timesteps, our method can perform well even in 100 steps, please refer to the highlighted regions for detailed comparison. Although the quality is close to ADMMDiff, our method demands less time.}
    \label{fig:exp_different_timestep_combined}
\end{figure*}

\begin{figure*}[h]
    \centering
    \includegraphics[width=\linewidth]{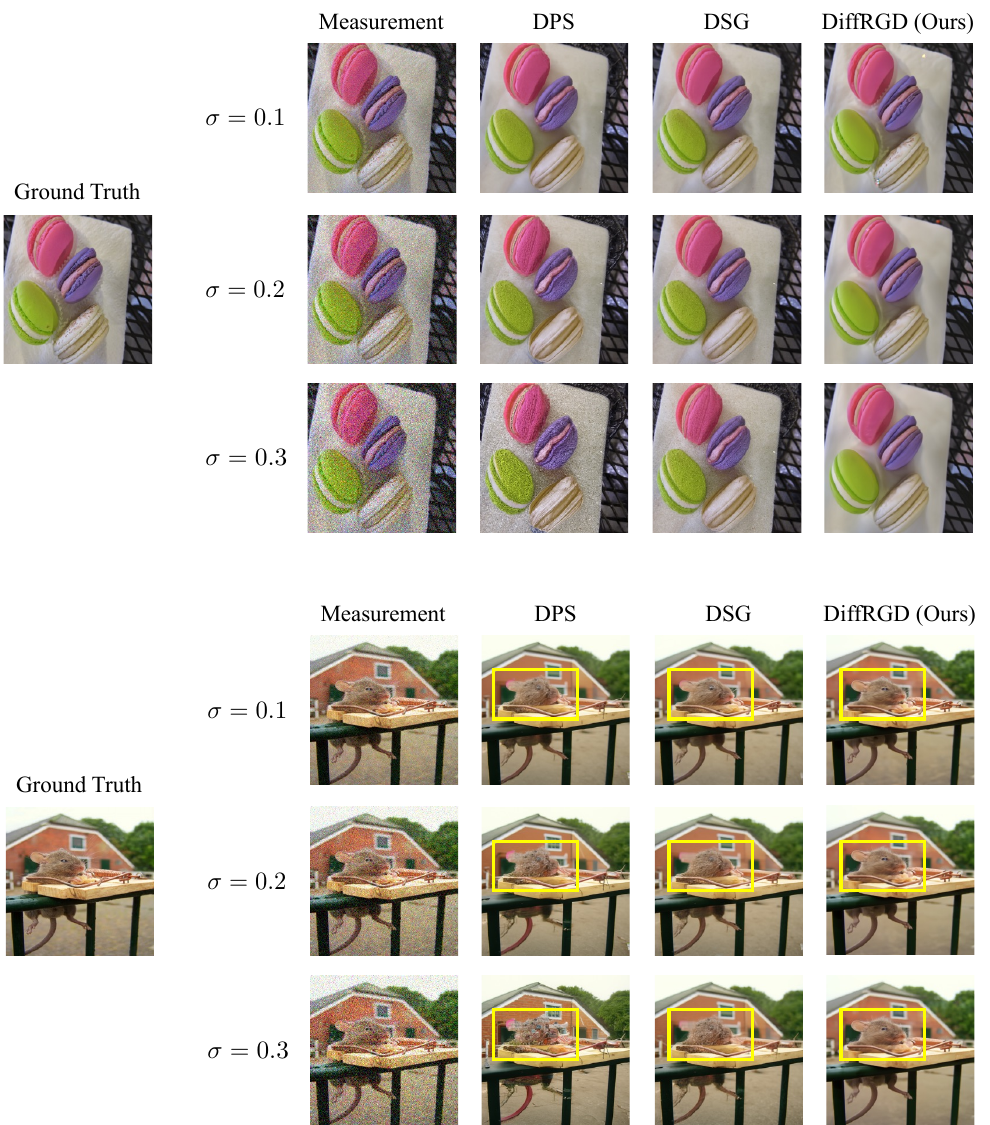}
    \caption{Qualitative comparison for image denoising tasks on ImageNet 256 $\times$ 256. Please refer to the highlighted regions for detailed comparison.}
    \label{fig:denoising}
\end{figure*}

\begin{figure*}[t]
    \centering
    \includegraphics[width=\linewidth]{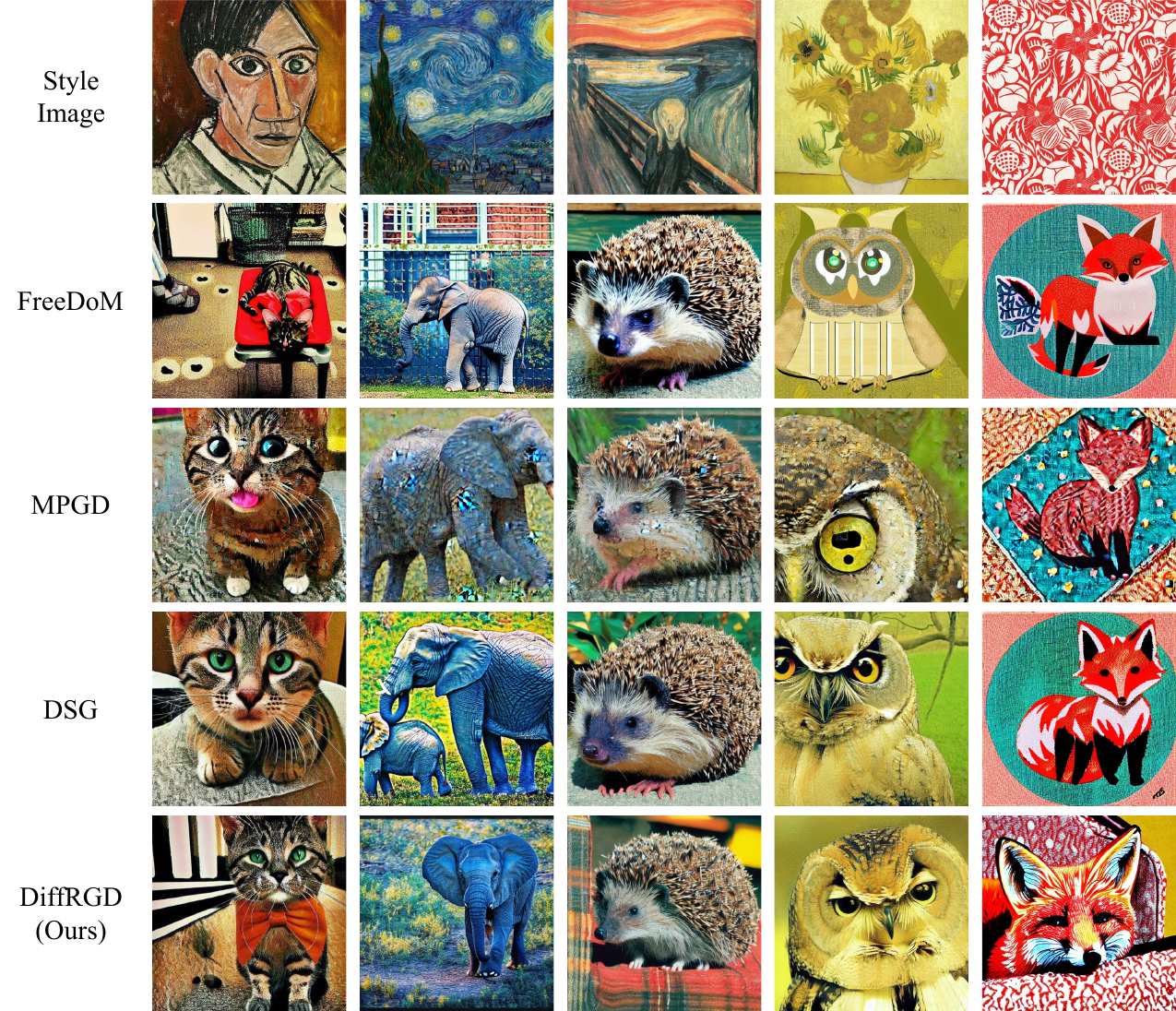}
    \caption{Qualitative comparison on the style-guided generation task using the Stable Diffusion v1.5 at a resolution of $512 \times 512$.}
    \label{fig:style_transfer}
\end{figure*}

\clearpage
\begin{table*}[t]
    \caption{Hyperparameters used in our image restoration tasks on FFHQ $256 \times 256$. The guidance interval specifies how frequently guidance is applied during the sampling process; for example, a value of $i$ indicates that guidance is performed once every $i$ steps, which helps accelerate the sampling.}
    \centering
    \resizebox{\linewidth}{!}{
    \begin{tabular}{llcccc}
        \toprule
        \textbf{Methods} & \textbf{Tasks} & \textbf{DDIM steps $T$} & \textbf{Guidance strength $\eta_t$} & \textbf{Guidance inner iteration $K$} & \textbf{Guidance interval $I$} \\
        \midrule
        \multirow{8}{*}{DPS}
            & Inpainting          & $1000$ & $1$ & $1$ & $1$ \\
            & Super-Resolution    & $1000$ & $1$ & $1$ & $1$ \\
            & Gaussian Deblurring & $1000$ & $1$ & $1$ & $1$ \\
            & Motion Deblurring   & $1000$ & $1$ & $1$ & $1$ \\
        \cmidrule(lr){2-6}
            & Inpainting          & $100$ & $1$ & $1$ & $1$ \\
            & Super-Resolution    & $100$ & $1$ & $1$ & $1$ \\
            & Gaussian Deblurring & $100$ & $1$ & $1$ & $1$ \\
            & Motion Deblurring   & $100$ & $1$ & $1$ & $1$ \\
        \midrule
        \multirow{8}{*}{MPGD} 
            & Inpainting          & $1000$ & $30 / \sqrt{\bar{\alpha}_t}$ & $1$ & $1$ \\
            & Super-Resolution    & $1000$ & $30 / \sqrt{\bar{\alpha}_t}$ & $1$ & $1$ \\
            & Gaussian Deblurring & $1000$ & $50 / \sqrt{\bar{\alpha}_t}$ & $1$ & $1$ \\
            & Motion Deblurring   & $1000$ & $50 / \sqrt{\bar{\alpha}_t}$ & $1$ & $1$ \\
        \cmidrule(lr){2-6}
            & Inpainting          & $100$ & $30 / \sqrt{\bar{\alpha}_t}$ & $1$ & $1$ \\
            & Super-Resolution    & $100$ & $30 / \sqrt{\bar{\alpha}_t}$ & $1$ & $1$ \\
            & Gaussian Deblurring & $100$ & $50 / \sqrt{\bar{\alpha}_t}$ & $1$ & $1$ \\
            & Motion Deblurring   & $100$ & $50 / \sqrt{\bar{\alpha}_t}$ & $1$ & $1$ \\
        \midrule
        \multirow{8}{*}{DSG} 
            & Inpainting          & $1000$ & $0.2$ & $1$ & $5$ \\
            & Super-Resolution    & $1000$ & $0.2$ & $1$ & $20$ \\
            & Gaussian Deblurring & $1000$ & $0.2$ & $1$ & $5$ \\
            & Motion Deblurring   & $1000$ & $0.2$ & $1$ & $5$ \\
        \cmidrule(lr){2-6}
            & Inpainting          & $100$ & $0.2$ & $1$ & $1$ \\
            & Super-Resolution    & $100$ & $0.1$ & $1$ & $2$ \\
            & Gaussian Deblurring & $100$ & $0.1$ & $1$ & $1$ \\
            & Motion Deblurring   & $100$ & $0.1$ & $1$ & $1$ \\
        \midrule
        \multirow{8}{*}{ADMMDiff} 
            & Inpainting          & $1000$ & $(\eta_t, \rho)=(1, 0.9)$ & $1$ & $1$ \\
            & Super-Resolution    & $1000$ & $(\eta_t, \rho)=(1, 0.9)$ & $1$ & $1$ \\
            & Gaussian Deblurring & $1000$ & $(\eta_t, \rho)=(1, 0.9)$ & $1$ & $1$ \\
            & Motion Deblurring   & $1000$ & $(\eta_t, \rho)=(1, 0.9)$ & $1$ & $1$ \\
        \cmidrule(lr){2-6}
            & Inpainting          & $100$ & $(\eta_t, \rho)=(2.4, 0.5)$ & $10$ & $1$ \\
            & Super-Resolution    & $100$ & $(\eta_t, \rho)=(2.4, 0.5)$ & $10$ & $1$ \\
            & Gaussian Deblurring & $100$ & $(\eta_t, \rho)=(2.4, 0.5)$ & $10$ & $1$ \\
            & Motion Deblurring   & $100$ & $(\eta_t, \rho)=(2.4, 0.5)$ & $10$ & $1$ \\
        \midrule
        \multirow{8}{*}{DiffRGD} 
            & Inpainting          & $1000$ & $5$   & $3$ & $5$ \\
            & Super-Resolution    & $1000$ & $5$   & $3$ & $10$ \\
            & Gaussian Deblurring & $1000$ & $5$   & $3$ & $5$ \\
            & Motion Deblurring   & $1000$ & $2.5$ & $3$ & $4$ \\
        \cmidrule(lr){2-6}
            & Inpainting          & $100$ & $5$   & $3$ & $1$ \\
            & Super-Resolution    & $100$ & $5$   & $3$ & $1$ \\
            & Gaussian Deblurring & $100$ & $7$   & $3$ & $1$ \\
            & Motion Deblurring   & $100$ & $2.5$ & $3$ & $1$ \\
        \bottomrule
    \end{tabular}
    }
    \label{tab:hyperparams_ffhq}
\end{table*}

\begin{table*}[t]
    \caption{Hyperparameters used in our conditional generation task on CelebA-HQ $256 \times 256$.}
    \centering
    \resizebox{\linewidth}{!}{
    \begin{tabular}{llcccc}
        \toprule
        \textbf{Methods} & \textbf{Guidance} & \textbf{DDIM steps $T$} & \textbf{Guidance strength $\eta_t$} & \textbf{Guidance inner iteration $K$} & \textbf{Guidance interval $I$} \\
        \midrule
        \multirow{3}{*}{FreeDoM}
            & Segmentation & $100$ & $10$  & $1$ & $1$ \\
            & Sketch       & $100$ & $0.5$ & $1$ & $1$ \\
            & FaceID       & $100$ & $5$   & $1$ & $1$ \\
        \midrule
        \multirow{3}{*}{DSG} 
            & Segmentation & $100$ & $0.1$   & $1$ & $1$ \\
            & Sketch       & $100$ & $0.08$  & $1$ & $1$ \\
            & FaceID       & $100$ & $0.025$ & $1$ & $1$ \\
        \midrule
        \multirow{3}{*}{ADMMDiff} 
            & Segmentation   & $100$ & $(\eta_t, \rho)=(5, 0.1)$ & $10$  & $1$ \\
            & Sketch         & $100$ & $(\eta_t, \rho)=(2, 0.5)$ & $10$ & $1$ \\
            & FaceID         & $100$ & $(\eta_t, \rho)=(3, 0.2)$ & $10$ & $1$ \\
        \midrule
        \multirow{3}{*}{DiffRGD} 
            & Segmentation & $100$ & $30$ & $3$ & $1$ \\
            & Sketch       & $100$ & $2$  & $3$ & $1$ \\
            & FaceID       & $100$ & $5$  & $3$ & $1$ \\
        \bottomrule
    \end{tabular}
    }
    \label{tab:hyperparams_celeb_hq}
\end{table*}

\end{document}